%% file: main.tex
\def\year{2021}\relax
\documentclass[letterpaper]{article} 
\usepackage{aaai21}  
\usepackage{times}  
\usepackage{helvet} 
\usepackage{courier}  
\usepackage[hyphens]{url}  
\usepackage{graphicx} 
\usepackage{natbib}  
\usepackage{caption} 

\usepackage[switch]{lineno}  %

\include{header}

\title{High-Confidence Off-Policy (or Counterfactual) Variance Estimation}
\author {
        Yash Chandak \,\,\,\,
        Shiv Shankar \,\,\,\, Philip S. Thomas\\
}
\affiliations {
    University of Massachusetts \\
    \{ychandak, sshankar, pthomas\}@cs.umass.edu
}
\begin{document}

\maketitle

\begin{abstract}
Many sequential decision-making systems leverage data collected using prior policies to propose a new policy.
For critical applications, it is important that high-confidence guarantees on the new policy's behavior are provided before deployment, to ensure that the policy will behave as desired.
Prior works have studied high-confidence off-policy estimation of the \emph{expected} return, however, high-confidence off-policy estimation of the \emph{variance} of returns can be  equally critical for high-risk applications. 
In this paper we tackle the previously open problem of estimating and bounding, with high confidence, the variance of returns from off-policy data. 
%
\end{abstract}

\section{Introduction}

\textit{Reinforcement learning} (RL) has emerged as a promising method for solving sequential decision-making problems \cite{SuttonBarto2}.
Deploying RL to real-world applications, however, requires additional consideration of reliability, which has been relatively understudied.
Specifically, it is often desirable to provide high-confidence guarantees on the behavior of a given policy, \textit{before} deployment, to ensure that the policy will behave as desired.
%

Prior works in RL have studied the problem of providing high-confidence guarantees on the \textit{expected} return of an evaluation policy, $\pi$, using only data collected from a currently deployed policy called the \textit{behavior policy}, $\beta$ \cite{thomas2015high,Hanna2017Bootstrapping,kuzborskij2020confident}.
Analogously, researchers have also studied the problem of \textit{counter-factually} estimating and bounding the average treatment effect, with high confidence, using data from past treatments \cite{bottou2013counterfactual}.   
While these methods present important contributions towards developing practical algorithms, real-world problems may require additional consideration of the \textit{variance} of returns (effect) under any new policy (treatment) before it can be deployed responsibly.
%
%

For applications that have high stakes in the terms of financial cost or public well-being, only providing guarantees on the mean outcome might \textit{not} be sufficient.
\textit{Analysis of variance} (ANOVA) has therefore become a \textit{de-facto} standard for many industrial and medical applications \cite{tabachnick2007experimental}.
Similarly, analysis of variance can inform numerous real-world applications of RL.
For example, (a) analysing the variance of outcomes in a robotics application \citep{kuindersma2013variable}, 
(b) ensuring that the variance of outcomes for a medical treatment is not high, (c) characterizing the variance of customer experiences for a recommendation system \cite{teevan2009using}, or (d) limiting the variability of the performance of an autonomous driving system \cite{montgomery2007introduction}.

More generally, variance estimation can be used to account for risk in decision-making by designing objectives that maximize the mean of returns but minimize the variance of returns \citep{sato2001td,di2012policy, la2013actor}. 
Variance estimates have also been shown to be useful for automatically adapting hyper-parameters, like the exploration rate \citep{sakaguchi2004reliability} or  $\lambda$ for eligibility-traces \citep{white2016greedy}, and might also inform other methods that depend on the entire distribution of returns \citep{bellemare2017distributional,dabney2017distributional}. 

\begin{figure}
    \centering
    \includegraphics[width=0.375\textwidth]{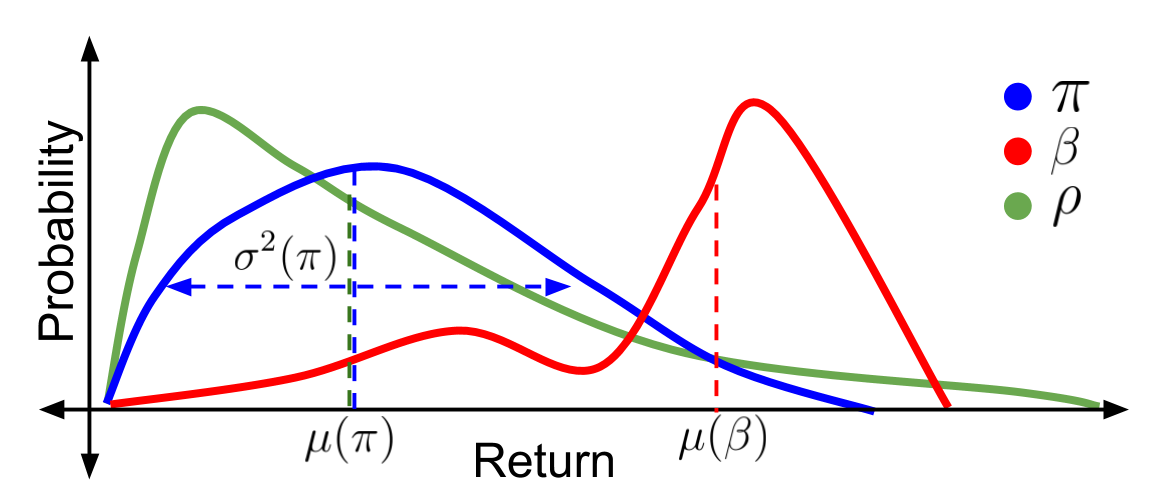}
    \caption{Illustrative example of the distributions of returns from a behavior policy $\beta$, and evaluation policy $\pi$, along with the \emph{importance weighted returns} $\rho$, discussed later. 
    Given trajectories from the behavior policy $\beta$, we aim to estimate and bound the variance, $\sigma^2(\pi)$, of returns under an evaluation policy $\pi$, with high confidence. 
    Note that the distribution of \textit{importance-weighted} returns $\rho$ has the mean value $\mu(\pi)$, but might have variance not equal to $\sigma^2(\pi)$.  }
    \label{fig:OVE}
\end{figure}

Despite the wide applicability of variance analysis, estimating and bounding the variance of returns with high confidence, using only off-policy data, has remained an understudied problem.  
%
%
%
%
In this paper, we first formalize the problem statement; 
an illustration of which is provided in Figure \ref{fig:OVE}.
We show that the typical use of \textit{importance sampling} (IS) in RL only corrects for the mean, and so it does not directly provide unbiased off-policy estimates of variance. 
We then present an off-policy estimator of the variance of returns that uses 
IS twice, together with a simple double-sampling technique.
To reduce the variance of the estimator, 
we extend the per-decision IS technique \cite{precup2000eligibility} to off-policy variance estimation.
Building upon this estimator, we provide confidence intervals for the variance using (a) concentration inequalities, and (b) statistical bootstrapping.\\
%


\noindent\textbf{Advantages: }
The proposed variance estimator has several advantages: (a) it is a model-free estimator and can thus be used irrespective of the environment complexity, (b) it requires only off-policy data and can therefore be used \textit{before} actual policy deployment, (c) it is unbiased and  consistent.
For high-confidence guarantees, (d) we provide both upper and lower confidence intervals for the variance that have guaranteed coverage (that is, they hold with any desired confidence level and without requiring false assumptions), 
and (e) we also provide bootstrap confidence intervals, which are approximate but often more practical.\\

\noindent \textbf{Limitations: } The proposed off-policy estimator of the variance relies upon IS and thus inherits its limitations. Namely,  (a)  it requires knowledge of  the action probabilities from the behavior policy $\beta$, (b)  it requires that the support of the trajectories under the evaluation policy $\pi$ is a subset of the support under the behavior policy $\beta$, and (c) the variance of the estimator scales exponentially with the length of the trajectory \cite{guo2017using,liu2018breaking}.
%

\section{Background and Problem Statement}
A \textit{Markov decision process} (MDP) is a tuple $(\mathcal S, \mathcal A, \mathcal P, \mathcal R, \gamma, d_0)$, where $\mathcal S$ is the set of states, $\mathcal A$ is the set of actions, $\mathcal P$ is the transition function, $\mathcal R$ is the reward function, $\gamma \in [0,1)$ is the discount factor, and $d_0$ is the starting state distribution.\footnote{
We formulate the problem in terms of MDPs, but it can analogously be formulated in terms of \textit{structural causal models.} \cite{pearl2009causality}.
%
%
For simplicity, we consider finite  states and actions, but our results extend to POMDPs (by replacing states with \textit{observations}) and to continuous states and actions (by appropriately replacing summations with integrals), and to infinite horizons ($T \coloneqq \infty$).}
%
%
A policy $\pi$ is a distribution over the actions  conditioned on the state, i.e., $\pi(a|s)$ represents the probability of taking action $a$ in state $s$.
We assume that the MDP has finite horizon $T$, after which any action leads to an absorbing state $S_{(\infty)}$. 
In general, we will use subscripts with parentheses for the timestep and subscript without parentheses to indicate the episode number.
Let $R_{i(j)} \in [R_\texttt{min}, R_\texttt{max}]$ represent the reward observed at timestep $j$ of the episode $i$. 
Let the random variable $G_i \coloneqq \sum_{j=0}^T \gamma^j R_{i(j)}$ be the \textit{return} for episode $i$.
Let $c \coloneqq (1 - \gamma^T)/(1-\gamma)$ so that the minimum and the maximum returns possible are $G_\texttt{min} \coloneqq c R_\texttt{min}$ and $G_\texttt{max} \coloneqq c R_\texttt{max}$, respectively.
Let $\mu(\pi) \coloneqq \mathbb{E}_\pi[G]$ be the expected return, and $\sigma^2(\pi) \coloneqq \mathbb{V}_\pi[G]$ be the variance of returns, where the subscript $\pi$ denotes that the trajectories are generated using policy $\pi$.
Let $\mathcal H^\pi_{(i):(j)}$ be the set of all possible trajectories for a policy $\pi$, from timestep $i$ to timestep $j$.
Let $H$ denote a complete trajectory: $(S_{(0)}, A_{(0)}, \Pr(A_{(0)}|S_{(0)}), R_{(0)}, S_{(1)}, ..., S_{(\infty)})$, where $T$ is the horizon length, and $S_{(0)}$ is sampled from $d_0$.
Let $\mathcal D$ be a set of $n$ trajectories $\{H_i\}_{i=1}^n$ generated using \textit{behavior} policies $\{\beta_i\}_{i=1}^n$, respectively.
Let $\rho_i(0,T) \coloneqq \prod_{j=0}^{T} \frac{\pi(A_{i(j)}|S_{i(j)})}{\beta_i(A_{i(j)}|S_{i(j)})}$ denote the product of \textit{importance ratios} from timestep $0$ to $T$.
For brevity, when the range of timesteps is not necessary, we write $\rho_i \coloneqq \rho_i(0,T)$. 
Similarly, when referring to $\rho_i$ for an arbitrary $i \in \{1,\dotsc,n\}$, we often write $\rho$.
With this notation, we now formalize the \textit{off-policy variance estimation} (OVE) and the \textit{high-confidence off-policy variance estimation} (HCOVE) problems.

\noindent
\paragraph{OVE Problem:}  Given a set of trajectories $\mathcal D$ and an \textit{evaluation} policy $\pi$, we aim to find an estimator $\hat \sigma_n^2$ that  is both an unbiased and consistent estimator of $\sigma^2(\pi)$, i.e.,
\begin{align}
    \mathbb{E}[\hat \sigma_n^2] &= \sigma^2(\pi), &
    \hat \sigma_n^2 &\overset{\text{a.s.}}{\longrightarrow} \sigma^2(\pi).
\end{align}

\noindent
\paragraph{HCOVE Problem:}  Given a set of trajectories $\mathcal D$, an \textit{evaluation} policy $\pi$, and a confidence level $1 - \delta$, we aim to find a confidence interval $\mathcal C \coloneqq [v^{\texttt{lb}}, v^{\texttt{ub}}]$, such that
\begin{align}
    \Pr \left(\sigma^2(\pi) \in \mathcal C\right) \geq 1 - \delta.
\end{align}

\begin{rem}
It is worth emphasizing that the OVE problem is about estimating the variance of returns, and \textit{not} the variance of the estimator of the mean of returns.
\end{rem}

These problems would not be possible to solve if the trajectories in $\mathcal D$ are not informative about the trajectories that are possible under $\pi$.
For example, if $\mathcal D$ has no trajectory that could be observed if policy $\pi$ were to be executed, then $\mathcal D$ provides little or no information about the possible outcomes under $\pi$.
To avoid this case, 
%
we make the following  common assumption 
\cite{precup2000eligibility}, which is satisfied if $(\beta_i(a|s)=0)\implies(\pi(a|s) = 0)$ for all $s \in \mathcal S, a \in \mathcal A,$ and $i \in \{1,\dotsc,n\}$. 
\begin{ass} The set $\mathcal D$ contains independent trajectories generated using behavior policies $\{\beta_i\}_{i=1}^n$, such that 
$$\forall i, \,\, \mathcal H^\pi_{(0):(T)} \subseteq \mathcal H^{\beta_i}_{(0):(T)}.$$
\thlabel{ass:supp}
\end{ass} 
%
%

The methods that we derive, and IS methods in general, do not require 
complete knowledge of $\{\beta_i\}_{i=1}^n$ (which might be parameterized using deep neural networks and might be hard to store). 
Only the probabilities, $\beta_i(a|s)$, for states $s$ and actions $a$ present in $\mathcal D$ are required. 
For simplicity, we restrict our notation to a single behavior policy $\beta$, such that $\forall i, \,\, \beta_i = \beta$.

\section{Na\"{i}ve Methods}
In the \textit{on-policy} setting, computing an estimate of $\mu(\pi)$ or $\sigma^2(\pi)$ is trivial---sample $n$ trajectories using $\pi$ and compute the sample mean or variance of the observed returns, $\{G_i\}_{i=1}^n$.
In the \textit{off-policy} setting, under \thref{ass:supp},  the sample mean  $\hat \mu \coloneqq \frac{1}{n}\sum_{i=1}^n \rho_i G_i$ of the \emph{importance weighted returns} $\{\rho_i G_i\}_{i=1}^n$, is an unbiased estimator of $\mu(\pi)$ \cite{precup2000eligibility}, i.e., 
$\mathbb{E}_\beta [\hat \mu]=\mu(\pi)$. 
Similarly, one natural way to estimate $\sigma^2(\pi)$ in the off-policy setting might be to compute the sample variance (with Bessel's correction) of the importance sampled returns $\{\rho_i G_i\}_{i=1}^n$,
\begin{align}
    \hat \sigma_n^{2!!} \coloneqq \frac{1}{n-1}\sum_{i=1}^n \left (\rho_i G_i - \frac{1}{n} \sum_{j=1}^n \rho_j G_j \right)^2. \label{eqn:naive1}
\end{align}

Unfortunately, $\hat \sigma_n^{2!!}$ is neither an unbiased nor consistent estimator of $\sigma^2(\pi)$, in general, as shown in the following properties. 
These properties also reveal that $\rho_i G_i$ only corrects the distribution for the mean and not for the variance, as depicted in Figure \ref{fig:OVE}. 
Also, note that all proofs are deferred to the appendix.
%

%

\begin{prop}
    Under \thref{ass:supp}, $\hat \sigma_n^{2!!}$ may be a biased estimator of $\sigma^2(\pi)$.
    That is, it is possible that $\mathbb{E}_\beta[\hat \sigma_n^{2!!}] \neq \sigma^2(\pi)$. \thlabel{prop:naive1biased}
\end{prop}

\begin{prop}
    Under \thref{ass:supp}, $\hat \sigma_n^{2!!}$ may not be a consistent estimator of $\sigma^2(\pi)$.
    That is, it is not always the case that 
    $\hat \sigma_n^{2!!} \overset{\text{a.s.}}{\longrightarrow} \sigma^2(\pi)$. \thlabel{prop:naive1inconsistent}
\end{prop}
%

%
%
%
%
%

Since the on-policy variance is $\mathbb{E}_\pi[(G - \mathbb{E}_\pi[G])^2]$, 
a natural alternative might be to construct an estimator that corrects the off-policy distribution for both the mean and the variance.
That is, using the equivalence
\begin{align}
   \mathbb{V}_\pi(G)  = \mathbb{E}_\pi[(G - \mathbb{E}_\pi[G])^2] =  \mathbb{E}_\beta[\rho (G - \mathbb{E}_\beta[\rho G])^2], \label{eqn:doublecorrection}
\end{align}
an alternative might be to use a plug-in estimator for  $\mathbb{E}_\beta[\rho (G - \mathbb{E}_\beta[\rho G])^2]$ (with Bessel's correction) as,%

\begin{align}
    \hat \sigma_n^{2!} \coloneqq \frac{1}{n-1}\sum_{i=1}^n  \rho_i \left ( G_i - \frac{1}{n} \sum_{j=1}^n \rho_j G_j \right)^2.
    \label{eqn:naive2}
\end{align}
While $\hat \sigma_n^{2!}$ turns out to be a consistent estimator,  it is still not an unbiased estimator of $\sigma^2(\pi)$.
We formalize this in the following properties.

\begin{prop}
    Under \thref{ass:supp}, $\hat \sigma_n^{2!}$ may be a biased estimator of $\sigma^2(\pi)$.
    That is, it is possible that $\mathbb{E}_\beta[\hat \sigma_n^{2!}] \neq \sigma^2(\pi)$. \thlabel{prop:naive2biased}
\end{prop}

\begin{prop}
    Under \thref{ass:supp}, $\hat \sigma_n^{2!}$ is a consistent estimator of $\sigma^2(\pi)$.
    That is, $\hat \sigma_n^{2!} \overset{\text{a.s.}}{\longrightarrow} \sigma^2(\pi)$. 
    \thlabel{prop:naive2consistent}
\end{prop}
%

    %
%
Before even considering confidence intervals for $\sigma^2(\pi)$, the lack of unbiased estimates from these na\"{i}ve methods leads to a basic question: \textit{How can we construct unbiased estimates of $\sigma^2(\pi)$?}
We answer this question in the following section.

\section{Off-Policy Variance Estimation}
Before constructing an unbiased estimator for $\sigma^2(\pi)$, we first  discuss one root cause for the bias of 
$\hat \sigma_n^{2!}$ and $\hat \sigma_n^{2!!}$. 
Notice that an expansion of \eqref{eqn:naive1} and \eqref{eqn:naive2} produces self-coupled importance ratio terms.
That is, terms consisting of $\rho_i^2$ and $\rho_i^3$.
While $\rho_i$ helps in correcting the distribution, its higher powers, $\rho_i^2$ and $\rho_i^3$, do not.

Expansion of \eqref{eqn:naive1} and \eqref{eqn:naive2} also results in cross-coupled importance ratio terms, $\rho_i\rho_j$, where $i \neq j$.
However, because $\mathbb E_\beta[\rho_i]=1$ for all $i \in \{1,\dotsc,n\}$ and because $\rho_i$ and $\rho_j$ are independent when $i \neq j$, these terms factor out in expectation. 
Hence, these terms do not create the troublesome higher powers of importance ratios. 
%

Based on these observations, we create an estimator that does not have any self-coupled importance ratio terms like $\rho_i^2$, but which may have $\rho_i\rho_j$ terms, where $i \neq j$.
To do so, we consider the alternate formulation of variance,
\begin{align}
   \mathbb{V}_\pi(G)  = \mathbb{E}_\pi[G^2] - \mathbb{E}_\pi[G]^2 =  \mathbb{E}_\beta[\rho G^2] - \mathbb{E}_\beta[\rho G]^2. \label{eqn:reform}
\end{align}
%
%
%
In \eqref{eqn:reform}, while a plug-in estimator of $\mathbb{E}_\beta[\rho G^2]$  would be unbiased and free of any self-coupled importance ratio terms, a plug-in estimator for $\mathbb{E}_\beta[\rho G]^2$ would neither be unbiased nor would it be free of $\rho^2_i$ terms.
To remedy this problem, we explicitly split the set of sampled trajectories into two mutually exclusive sets, $\mathcal D_1$ and $\mathcal D_2$, of equal sizes, and re-express $\mathbb{E}_\beta[\rho G]^2$ as $\mathbb{E}_\beta[\rho G]\mathbb{E}_\beta[\rho G]$, where the first expectation is estimated using samples from $\mathcal D_1$ and the second expectation is estimated using samples from $\mathcal D_2$.
Based on this \emph{double sampling} approach, we propose the following off-policy variance estimator, 

\begin{align}
    \hat \sigma_n^2 \coloneqq \frac{1}{n}\sum_{i=1}^{n} \rho_i G_i^2  - \left(\frac{1}{|\mathcal D_1|}\sum_{i=1}^{|\mathcal D_1|}\rho_i G_i \right) \left(\frac{1}{|\mathcal D_2|}\sum_{i=1}^{|\mathcal D_2|}\rho_i G_i \right). \label{eqn:proposed}
\end{align}
This simple data-splitting trick suffices to create, $\hat \sigma_n^2$, an off-policy variance estimator that is both unbiased and consistent.
We formalize this in the following theorems.

\begin{thm}
    Under \thref{ass:supp}, $\hat \sigma_n^2$ is an unbiased estimator of $\sigma^2(\pi)$.
    That is, $\mathbb{E}_\beta[\hat \sigma_n^2] = \sigma^2(\pi)$. \thlabel{thm:unbiased}
\end{thm}

\begin{thm}
    Under \thref{ass:supp}, $\hat \sigma_n^2$ is a consistent estimator of $\sigma^2(\pi)$.
    That is, 
    $\hat \sigma_n^2 \overset{\text{a.s.}}{\longrightarrow} \sigma^2(\pi)$. \thlabel{thm:consistent}
\end{thm}

\begin{rem}
  It is possible that $\hat \sigma_n^2$  results in negative values (see Appendix C for an example).
   One practical solution to avoid negative values for variance can be to define $\hat \sigma_n^{2+} \coloneqq \texttt{clip}(\hat \sigma_n^2, \texttt{min}=0, \texttt{max}=\infty)$.
  However, this may make $\hat \sigma_n^{2+}$ a biased estimator, i.e., $\mathbb{E}_\beta[\hat \sigma_n^{2+}] \neq \sigma^2(\pi)$. 
  Notice that this is the expected behavior of IS based estimators. 
  For example, the IS estimates of expected return can be smaller or larger than the smallest and largest possible returns when $\rho > 1$. 
  We refer the reader to the works by \citet{mchugh1968negative}, \citet{ anderson1965negative}, and \citet{ nelder1954interpretation} 
for other occurences of negative variance and its interpretations. 
\end{rem}


\section{Variance-Reduced Estimation of Variance}
Despite $\hat \sigma_n^2$ being both an unbiased and a consistent estimator of variance, the use of IS can make its variance high.
Specifically, the importance ratio $\rho$ 
may become unstable when its denominator, $\prod_{i=0}^{T} \beta(A_{(i)}|S_{(i)})$, is small. 

To mitigate variance, it is common in off-policy mean estimation to use \textit{per-decision importance sampling} (PDIS), instead of the full-trajectory IS, to reduce variance without incurring any bias \cite{precup2000eligibility}.
It is therefore natural to ask: Is it also possible to have something like PDIS for off-policy variance estimation?

Recall from \eqref{eqn:proposed} that the expectation of the terms inside the parentheses correspond to $\mathbb{E}_\beta[\rho G] = \mu(\pi)$, a term for which we can directly leverage the existing PDIS estimator,
    $\mathbb{E}_\beta[\rho G] =\mathbb{E}_\beta \left[ \sum_{i=0}^T  \rho(0, i)\gamma^i R_{(i)} \right]$.
%
Intuitively, PDIS leverages the fact that the probability of observing an individual reward at timestep $i$ only depends upon the probability of the trajectory up to timestep $i$.

However, the first term in the \emph{right hand side} (RHS) of \eqref{eqn:proposed} contains $G^2 = (\sum_{i=o}^T \gamma^i R_{(i)})^2$. 
Expanding this expression, we obtain 
self-coupled and cross-coupled \textit{reward} terms, $R_{(i)}^2$ and $R_{(i)}R_{(j)}$,
which makes PDIS not directly applicable.
%
In the following theorem we present a new estimator, \textit{coupled-decision importance sampling} (CDIS), which extends PDIS to handle these coupled rewards.

\begin{thm} Under \thref{ass:supp}, 
\begin{align}
    \mathbb{E}_\beta\left[\rho G^2 \right] = \mathbb{E}_\beta\left[\sum_{i=0}^{T} \sum_{j=0}^T \rho\left(0,\texttt{max}(i,j)\right) \gamma^{i+j} R_{(i)} R_{(j)}\right]. \label{eqn:cdis}
\end{align}
\thlabel{thm:cdis}
\end{thm}
Borrowing intuition from PDIS, CDIS leverages the fact that the probability of observing a coupled-reward, $R_{(i)}R_{(j)}$, only depends on the probability of the trajectory up to the $i$ or $j$th timestep, whichever is larger. 
Importance ratios beyond that timestep can therefore be discarded without incurring bias.
In Algorithm \ref{apx:Alg:2}, we combine both per-decision and coupled-decision IS to construct a variance-reduced estimator of $\sigma^2(\pi)$. 

	\begin{algorithm2e}[h]
		\textbf{Input:} {Set of trajectories $\mathcal D$} 
		\\
    		$\mathcal D_1, \mathcal D_2 \leftarrow \texttt{equal\_split}(\mathcal D) $
    		\\
    		$X = \frac{1}{|\mathcal D|} \sum\limits_{i=1}^{|\mathcal D|}\sum\limits_{j=0}^{T} \sum\limits_{k=0}^T \rho_i(0,\texttt{max}(j,k)) \gamma^{j+k} R_{i(j)} R_{i(k)}$ 
    		\\
    		$Y = \frac{1}{|\mathcal D_1|} \sum\limits_{i=1}^{|\mathcal D_1|} \sum\limits_{j=0}^T  \rho(0, j)\gamma^{j} R_{i(j)}$ 
    		\\
    		$Y' = \frac{1}{|\mathcal D_2|} \sum\limits_{i=1}^{|\mathcal D_2|} \sum\limits_{j=0}^T  \rho(0, j)\gamma^{j} R_{i(j)}$
    		\\
    		\textbf{Return} $X - YY'$ 
		\caption{Variance-Reduced Off-Policy Variance Estimator}
		\label{apx:Alg:2}  
	\end{algorithm2e}

\section{HCOVE using Concentration Inequalities}

In the previous section we found that the reformulation presented in  \eqref{eqn:reform} was helpful for creating a variance reduced off-policy variance estimator $\hat \sigma_n^2$.
In this section, we will again build upon \eqref{eqn:reform} to obtain a \textit{confidence interval} (CI) for $\sigma^2(\pi)$.
One specific advantage of \eqref{eqn:reform} is that it allows us to build upon existing concentration inequalities, which were developed for obtaining CIs for $\mu(\pi)$, to obtain a CI for $\sigma^2(\pi)$.

Before moving further, we define some additional notation.
For any random variable $X$, let $\operatorname{CI}^+(\mathbb{E}[X], \delta)$, $\operatorname{CI}_-(\mathbb{E}[X], \delta)$, and $\operatorname{CI}^+_-(\mathbb{E}[X], \delta)$ represent only upper, only lower, and both upper and lower $(1-\delta)$-confidence bounds  for $\mathbb{E}[X]$, respectively. 
That is, $\Pr(\operatorname{CI}^+(\mathbb{E}[X], \delta) \geq \mathbb{E}[X]) \geq 1 - \delta$, $\Pr(\operatorname{CI}_-(\mathbb{E}[X], \delta) \leq \mathbb{E}[X]) \geq 1 - \delta$, etc. 
For brevity, We will sometimes suppress CI's dependency on $\delta$.

With the above notation, we now establish a high-confidence bound on \eqref{eqn:reform}.
Recall that \eqref{eqn:reform} consists of one positive term $\mathbb{E}[\rho G^2]$ and a negative term $-\mathbb{E}[\rho G]^2$.
Therefore, given a confidence interval for both of these terms, the high-confidence upper bound for \eqref{eqn:reform} would be the high-confidence upper bound of $\mathbb{E}[\rho G^2]$ minus the high-confidence lower bound of $\mathbb{E}[\rho G]^2$, and vice-versa to obtain a high-confidence lower bound on \eqref{eqn:reform}.
That is, let $\delta_1, \delta_2, \delta_3$ and $\delta_4$ be some constants in $(0,0.5]$ such that $\delta/2 = \delta_1+\delta_2 = \delta_3+\delta_4$. 
The lower bound $v^{\texttt{lb}}$ and the upper bound $v^{\texttt{ub}}$ can be expressed as, 
\begin{align}
    v^{\texttt{lb}} &\coloneqq \operatorname{CI}_-\left(\mathbb{E}_\beta[\rho G^2], \delta_1 \right) - \operatorname{CI}^+\left(\mathbb{E}_\beta[\rho G]^2, \delta_2 \right),\label{eqn:lb}\\
    v^{\texttt{ub}} &\coloneqq  \operatorname{CI}^+\left(\mathbb{E}_\beta[\rho G^2], \delta_3 \right) - \operatorname{CI}_-\left(\mathbb{E}_\beta[\rho G]^2, \delta_4 \right).\label{eqn:ub}
\end{align}
For getting the desired CIs for the first terms in the RHS of \eqref{eqn:lb} and \eqref{eqn:ub}, notice that any method for obtaining a CI on the expected return, $\mathbb{E}_\beta[\rho G]$, can also be used to bound $\mathbb{E}_\beta[\rho G']$, where $G' \coloneqq G^2$.

For getting the desired CIs in the second term in the RHS of \eqref{eqn:lb} and \eqref{eqn:ub}, we perform \textit{interval propagation} \cite{jaulin2002interval}.
That is, given a high confidence interval for $\mathbb{E}[\rho G]$, since $\mathbb{E}[\rho G]^2$ is a quadratic function of $\mathbb{E}[\rho G]$, the upper bound for the value of $\mathbb{E}[\rho G]^2$ would be the \textit{maximum} of the squared values of the end-points of the interval for $\mathbb{E}[\rho G]$.
Similarly, the lower bound on $\mathbb{E}[\rho G]^2$ would be $0$ if the signs of upper and lower bounds for $\mathbb{E}[\rho G]$ are different, otherwise it would be the \textit{minimum} of the squared value of the end-points of the interval for $\mathbb{E}[\rho G]$.
An illustration of this concept is presented in Figure \ref{fig:boundprop}.

Using interval propagation, the resulting upper bound is
\begin{equation}\operatorname{CI}^+(\mathbb{E}_\beta[\rho G]^2) = \texttt{max}(\operatorname{CI}_-(\mathbb{E}_\beta[\rho G])^2, \operatorname{CI}^+(\mathbb{E}_\beta[\rho G] )^2),
\end{equation}
and the resulting high-confidence lower bound is 
$\operatorname{CI}_-(\mathbb{E}_\beta[\rho G]^2)=0$ if both $\operatorname{CI}_-(\mathbb{E}_\beta[\rho G]) \leq 0$ and $\operatorname{CI}^+(\mathbb{E}_\beta[\rho G]) \geq 0$, and $\operatorname{CI}_-(\mathbb{E}_\beta[\rho G]^2)=\texttt{min}(\operatorname{CI}^+(\mathbb{E}_\beta[\rho G] )^2, \operatorname{CI}_-(\mathbb{E}_\beta[\rho G] )^2)$ otherwise.
%
%
%
%
Notice that these upper and lower high-confidence bounds on $\mathbb E_\beta[\rho G]^2$ can always be reduced to $\operatorname{max}(G^2_\texttt{min},G^2_\texttt{max})$ (the maximum squared return under any policy) when they are larger. 
%
%
\begin{figure}
    \centering
    \includegraphics[width=0.45\textwidth]{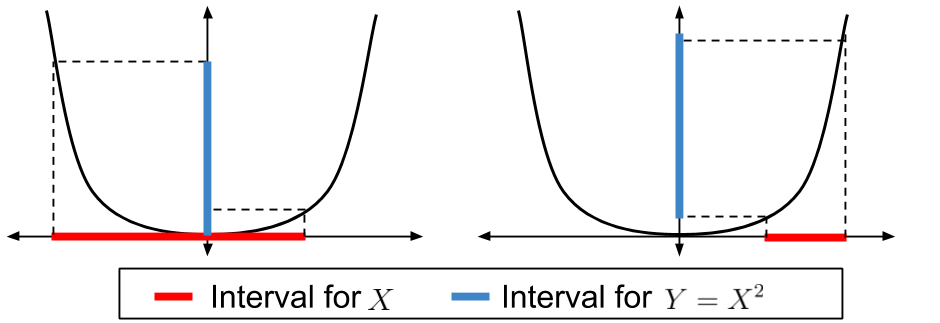}
    \caption{Two separate examples that show how interval propagation can be used to map confidence intervals for $X$ (in red) to confidence intervals for $Y=X^2$ (in blue).
    }
    \label{fig:boundprop}
\end{figure}

%
%


In the following theorem, we prove that the resulting confidence interval, $\mathcal C$, has \emph{guaranteed coverage}, i.e., that it holds with probability $1-\delta$. 
\begin{thm}[Guaranteed coverage] Under \thref{ass:supp}, if $(\delta_1 + \delta_2 + \delta_3 + \delta_4) \leq \delta$, then for the confidence interval $\mathcal C \coloneqq [v^{\texttt{lb}}, v^{\texttt{ub}}]$,
\begin{align}
    \Pr \left(\sigma^2(\pi) \in \mathcal C\right) \geq 1 - \delta.
\end{align}
\thlabel{thm:HCOVE1}
\end{thm}
%

\begin{rem}
    \thref{thm:HCOVE1} presents a two-sided interval. If only a lower bound or only an upper bound is required, then it suffices if only $(\delta_1 + \delta_2) \leq \delta$ or $(\delta_3 + \delta_4) \leq \delta$, respectively.
\end{rem}

\begin{rem}
     $\mathcal C$ can always be clipped via taking the intersection with the interval $[0,(G_\texttt{max}-G_\texttt{min})^2/4]$, since the variance will always be within this range (see Popoviciu's inequality for the deterministic upper bound on variance).
    %
\end{rem}

\subsection{A Tale of Long-Tails}
One important advantage of \thref{thm:HCOVE1} is that it constructs a CI, $\mathcal C$, for $\sigma^2(\pi)$ using \textit{any} concentration inequality that can be used to get CIs $\operatorname{CI}_+(\mathbb{E}_\beta[\rho G])$ and $\operatorname{CI}_-(\mathbb{E}_\beta[\rho G])$) for $\mu(\pi)$.
Hence, the tightness of $\mathcal C$ scales directly with the tightness of these existing off-policy policy evaluation methods for the expected discounted return. 
%
However, na\"{i}vely using common concentration inequalities can result in wide and uninformative CIs, as we discuss below. 
Therefore, in this section, we aim to establish a \textit{control-variate} that is designed to produce 
tighter CIs for $\sigma^2(\pi)$.

Typically, for a random variable $X \in [a, b]$, the width of the confidence interval for $\mathbb{E}[X]$ obtained using common concentration inequalities, like Hoeffding's \cite{hoeffding1994probability} or an empirical Bernstein inequality \cite{maurer2009empirical}, have a direct dependence on the range, $(b-a)$.
Unfortunately, as shown by \citet{thomas2015high}, IS based estimators may exhibit \textit{extremely} long tail behavior and can have a range in the order of $10^{10}$. 
For example, even if $\forall a \in \mathcal A$ and $\forall s \in \mathcal S$, if $\beta(a|s)>  0.1$, then the maximum possible importance weighted return of a ten timestep long trajectory can be on the order of $(1/0.1)^{10} = 10^{10}$ even when returns are normalized to the $[0,1]$ interval.  
Such a large range causes Hoeffding's inequality and empirical Bernstein inequalities to produce wide and uninformative confidence intervals, especially when the number of samples is not enormous. 

To construct a \textit{lower bound} for $\mathbb{E}[X]$, while being robust to the long tail, \citet{thomas2015high} notice that truncating the upper-tail of $X$ to a constant $c$ can only lower the expected value of $X$, i.e., for $X' \coloneqq \texttt{min}(X, c), \,\, \mathbb{E}[X'] \leq \mathbb{E}[X]$.
Therefore, $X_\texttt{lb}' \coloneqq \operatorname{CI}_-(\mathbb{E}[X'], \delta)$ is also 
a valid lower bound for $\mathbb{E}[X]$ and $\Pr(X_\texttt{lb}' \leq \mathbb{E}[X]) \geq 1- \delta$.
Additionally, truncating allows for significantly shrinking the range from $[a,b]$ to $[a,c]$, thereby effectively leading to a much tighter lower bound when $c$ is chosen appropriately.
%
%
For completeness, we review this bound in Appendix F.

While this bound was designed specifically for getting the \textit{lower bounds} required in \eqref{eqn:lb} and \eqref{eqn:ub}, it cannot be na\"{i}vely used to get the \textit{upper bounds}.
As $\mathbb{E}[X']\leq \mathbb{E}[X]$, the upper bound  $X_\texttt{ub}' \coloneqq \operatorname{CI}^+(\mathbb{E}[X'], \delta)$, may not be a valid upper bound for $\mathbb{E}[X]$ and $\Pr(X_\texttt{ub}' \geq \mathbb{E}[X]) \not\geq 1- \delta$.
%
A natural question is then: How can an upper bound be obtained that is robust to the long upper tail?

To answer this question, notice that if instead of the upper tail, the \textit{lower} tail of the distribution was long, then the upper bound constructed after truncating the lower tail would still be valid.
%
%
Therefore, we introduce a control-variate $\xi$ which can be used to switch the tails of the distribution of an IS based estimator, such that both upper and lower valid bounds can be obtained using the resulting distribution.
We formalize this in the following theorem.
\begin{thm}
    Let $X$ be either $G$ or $G^2$, then for any $\delta \in (0, 0.5]$ and a fixed constant $\xi$,
    \begin{align}
         \operatorname{CI}^+_-(\mathbb{E}_\beta[\rho X], \delta) = \operatorname{CI}^+_-(\mathbb{E}_\beta[\rho (X - \xi)], \delta) + \xi. 
    \end{align}
    \thlabel{prop:control}
\end{thm}

 
 \begin{rem}
    When $\xi$ is set to be the \textit{maximum} value that $X$ can take, then the random variable $\rho(X - \xi)$ will have an upper bound of $0$ and a long lower tail since $\rho \geq 0$ and $(X -\xi) \leq 0$.
    Similarly, when $\xi$ is set to be the \textit{minimum} value that $X$ can take, then the random variable $\rho(X-\xi)$ will have a lower bound of $0$ and a long upper tail. When a two-sided interval is required, two different estimators need to be constructed using the values for $\xi$ discussed above.
 \end{rem}

\thref{prop:control} allows us to control the tail-behavior such that the tight bounds presented by \citet{thomas2015high} can be leveraged to obtain \textit{both}  valid upper and valid lower high-confidence bound.
However, \thref{prop:control} still makes use of the full  trajectory importance ratio $\rho$, which can result in high-variance and inflate the confidence intervals.

To mitigate the above problem as well, we combine the variance reduction property of per-decision and coupled-decision IS offered by \thref{thm:cdis}, and the control over the tail behavior offered by \thref{prop:control}, and present the following theorem (see Appendix F for the complete algorithm). 
\begin{thm} Under \thref{ass:supp}, for any $\delta \in [0. 0.5]$, let $\xi_R \coloneqq \texttt{max}(R_\texttt{min}^2, R_\texttt{max}^2)$ and $\xi_G \coloneqq \texttt{max}(G_\texttt{min}^2, G_\texttt{max}^2)$ then
\begin{align}
    X \coloneqq& \sum_{i=0}^{T} \sum_{j=0}^T \rho\left(0,\texttt{max}(i,j)\right) \gamma^{i+j} \left( R_{(i)} R_{(j)} - \xi_R \right),
    \\
    Y \coloneqq& \sum_{i=0}^T  \rho(0, i)\gamma^i \left( R_{(i)} - R_\texttt{max}\right),
\end{align}
then $\Pr(X \leq 0) = \Pr(Y \leq 0) =1$, and
\begin{align}
    \operatorname{CI}^+\left(\mathbb{E}_\beta[\rho G^2], \delta\right) &= \operatorname{CI}^+\left(\mathbb{E}_\beta[X], \delta\right) + \xi_G,
    \\
    \operatorname{CI}^+\left(\mathbb{E}_\beta[\rho G], \delta\right) &= \operatorname{CI}^+\left(\mathbb{E}_\beta[Y], \delta\right) + G_\texttt{max}.
\end{align}
\end{thm}


\begin{rem}
    For a lower bound on $\mathbb{E}_\beta[\rho G^2]$, notice that $\rho G^2 \geq 0$ always and thus the lower bound on $\mathbb{E}_\beta[X]$, where $\xi_R$ and $\xi_G$ are set to $0$, can be used.
    Lower bound on $\mathbb{E}_\beta[\rho G]$ can be constructed by using the lower bound on $\mathbb{E}_\beta[Y]$, when $R_\texttt{max}$ and $G_\texttt{max}$ are replaced by $R_\texttt{min}$ and $G_\texttt{min}$.
\end{rem}


\begin{rem}
    If some trajectories have horizon length $t < T$, then they must be appropriately \textit{padded} to ensure that $\forall i \in [t+1, T], \,\, \rho(0,i) = \rho(0,t)$ and $R_{(i)} = 0$, such that in expectation the total amount added/subtracted by the control variate is zero.
\end{rem}

\section{HCOVE using Statistical Bootstrapping}

Bootstrap is a popular non-parametric technique for finding \textit{approximate} confidence intervals \cite{efron1994introduction}.
The core idea of bootstrap is to re-sample the observed data $\mathcal D$ and construct $B$ pseudo-datasets $\{ \mathcal D^*_i\}_{i=1}^B$ in a way such that each $\mathcal D^*_i$  resembles a draw from the true underlying \textit{data generating process}.
With each pseudo-data $\mathcal D_i$, an unbiased pseudo-estimate of a desired sample statistic can be created.
For our problem, this statistic corresponds to $\hat \sigma_n^{2*}$, the estimate of $\sigma^2(\pi)$ obtained using \eqref{eqn:proposed}.
Thereby, leveraging the entire set of pseudo-data $\{ \mathcal D^*_i\}_{i=1}^B$, an empirical distribution for the estimates of the variance $\{\hat \sigma^{2*}_{n,i}\}_{i=1}^B$ can be obtained.
This empirical distribution approximates the true distribution of $\hat \sigma_n^2$ and can thus be leveraged to obtain CIs for $\sigma^2(\pi)$ using the \textit{percentile} method, the \textit{bias-corrected and accelerated} (BCa) method, etc. \cite{diciccio1996bootstrap}.
A drawback of bootstrap is the increased computational cost required for re-sampling and analysing $B$ pseudo data-sets.
Further, the CIs obtained from bootstrap are only approximate, meaning that they can fail with more than $\delta$ probability.
However, the primary advantage of using bootstrap is that it provides much tighter CIs, as compared to the ones obtained using concentration inequalities, and hence can be more informative for certain applications in practice.

Let $\hat{ \mathcal C}$ be the approximate  interval for $\sigma^2(\pi)$, for a given confidence $\delta$, obtained using bootstrap (see Appendix F for the complete algorithm).
Then under the following assumption on the higher-moments of $\hat \sigma_n^2$, we directly leverage the results for bootstrap to obtain an error-rate for $\hat {\mathcal{C}}$. 
\begin{ass}
    The third moment of $\hat \sigma_n^2$ is bounded. That is, $\exists c_1< \infty$ such that $ \mathbb{E}_\beta[(\hat \sigma_n^2 - \mathbb{E}_\beta[\hat \sigma_n^2])^3] < c_1$. 
    \thlabel{ass:bounded}
\end{ass}
\thref{ass:bounded} is a typical requirement for bootstrap methods \citep{efron1994introduction}.
\thref{ass:bounded} can easily be satisfied by commonly used entropy regularized behavior policies that ensure that $\exists c_2 > 0$ such that $\forall a \in \mathcal A, \forall s \in \mathcal S, \,\, \beta(a|s) \geq c_2$.
This would ensure that the importance ratio $\rho \leq 1/(c_2^T)$ , and because $G \leq G_{\texttt{max}}$, $\rho G$ and $\rho G^2$ would also be bounded.
This ensures that $\hat \sigma_n^2$ is bounded, and therefore all its moments are also bounded, as required by \thref{ass:bounded}. 
%
%
We formalize the asymptotic correctness of bootstrap confidence intervals in the following theorem. 

\begin{thm} Under \thref{ass:supp,ass:bounded}, the confidence interval $\hat {\mathcal C}$ has a finite sample error of $O(n^{-1/2})$. That is,
\begin{align}
    \Pr \left(\sigma^2(\pi) \in \hat {\mathcal C}\right) \geq 1 - \delta - O\left(n^{-\frac{1}{2}}\right).
\end{align}
\thlabel{thm:boot}
\end{thm}

\begin{rem}
    Other variants of bootstrap (Bootstrap-t, BCa, etc.) can also be used, which typically offer \textit{higher order refinement} by reducing the finite sample error-rate to $O(n^{-1})$ \cite{diciccio1996bootstrap}.
\end{rem}
%

\section{Related Work}

\begin{figure*}[!ht]
    \centering
    \includegraphics[width=0.33\textwidth]{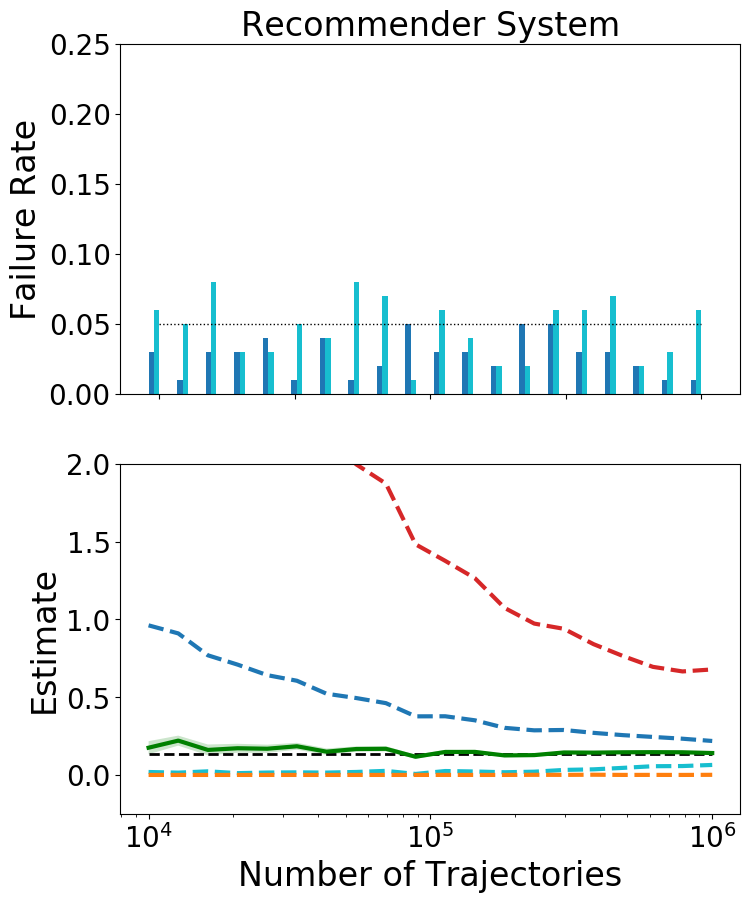} 
    \includegraphics[width=0.33\textwidth]{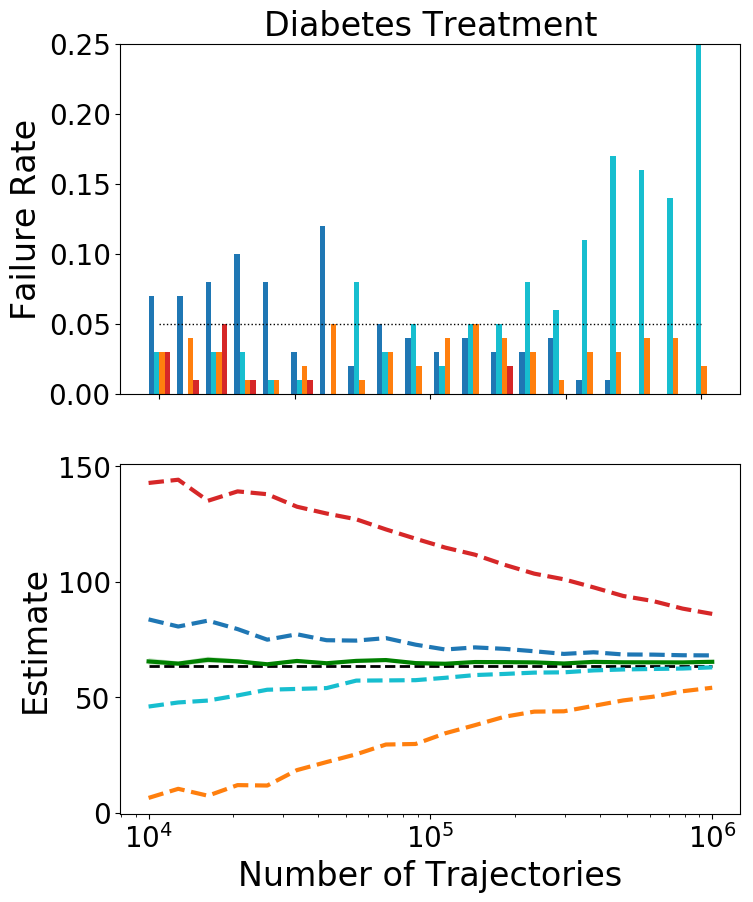}
    \includegraphics[width=0.33\textwidth]{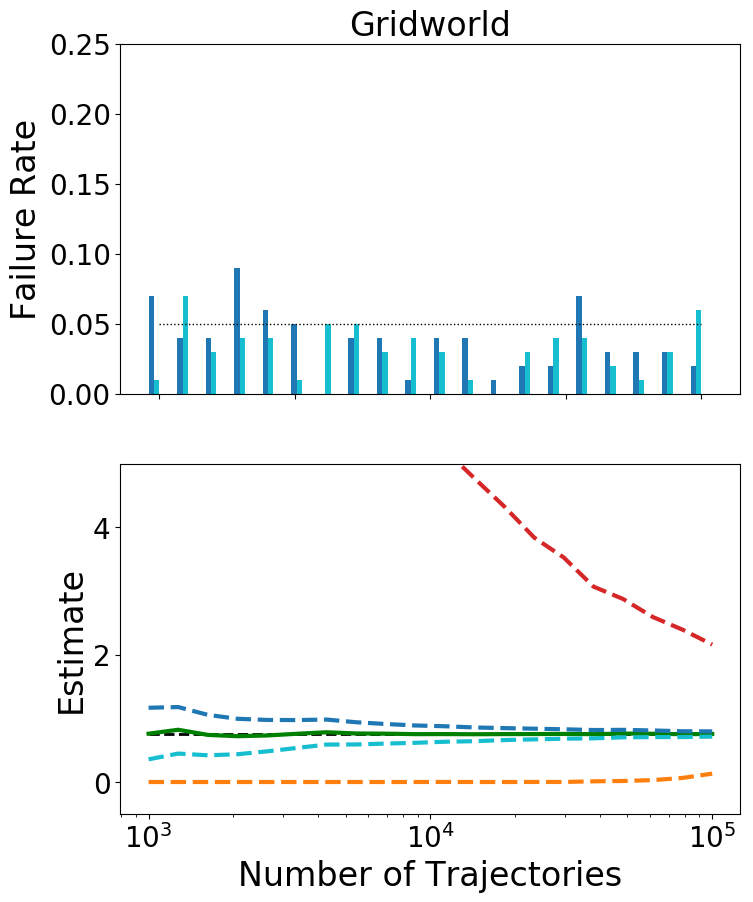}
    \\
    \includegraphics[width=0.65\textwidth]{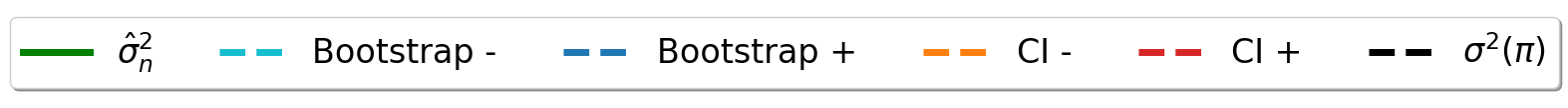}
    \caption{Experimental results using $100$ trials. (Top) Empirically observed fraction (out of 100 trials) for which the computed confidence interval did not include the actual variance,
    for the given number of trajectories (plotted on the shared horizontal axis), for the proposed upper and lower high-confidence bounds that were constructed using concentration inequalities (labeled as: $\operatorname{CI}+$, and $\operatorname{CI}-$) and using bootstrap (labeled as: Bootstrap $+$, and Bootstrap $-$).
    The color of the bars refer to the legend, and these bars should ideally be below the line representing the confidence level $\delta = 0.05$.
    (Bottom) The dashed, colored, lines represent the value of the respective high-confidence bounds, constructed with confidence level $1-\delta$ each.
    %
    The green line represents the value of our proposed estimator $\hat \sigma_n^2$ and the shaded area around it (almost negligible) corresponds to the standard error.
    Black dashed line represents the true variance, $\sigma^2(\pi)$.
    The unbiased and consistent property of $\hat \sigma_n^2$ can be visualized by comparing it with $\sigma^2(\pi)$.
    Notice that as the bootstrap confidence interval $\hat{\mathcal C}$ is only approximate, it can fail with more that $\delta$ probability. 
    In comparison, confidence interval $\mathcal C$ obtained using concentration inequalities provide guaranteed coverage.
    However, as clear from the plots, $\mathcal C$ can be conservative, while $\hat{\mathcal{C}}$ provides a tighter interval.
    }
    \label{fig:exps}
\end{figure*}

When samples are from the distribution whose variance needs to be estimated, then under the assumption that the distribution is normal, the $\chi^2$ distribution can be used for providing CIs for the variance.
Effects of non-normality on tests of significance were first analyzed by \citet{pearson1929distribution} and has led to a large body of literature on variance tests \cite{pearson1931analysis, box1953non, levene1960robust}.
%
Various modifications to $\chi^2$ tests have also been proposed to be robust against samples from non-normal distributions \cite{subrahmaniam1966some,garcia2006chi,pan1999levene, lim1996comparison}. 
%
%
%
The statistical bootstrap approach used in this paper to obtain bounds on the variance is closest to the bootstrap test developed by \citet{shao1990bootstrap}.
However, all of these methods are analogous to on-policy variance analysis.

In the context of RL, \citet{sobel1982variance} first introduced Bellman operators for the second moment and combined it with the first moment to compute the variance.
Temporal difference (TD) style algorithms have been subsequently developed for estimating the variance of returns \citep{tamar2016learning,la2013actor,white2016greedy,sherstan2018comparing}.
However, such TD methods might suffer from potential instabilities when used with function approximators and off-policy data \citep{SuttonBarto2}.
Policy gradient style algorithms have also been developed for finding policies that optimize variance related objectives  \citep{di2012policy,tamar2013variance},
however, these are limited to the on-policy setting.
We are not aware of any work in the RL literature that provides unbiased and consistent off-policy variance estimators, nor high-confidence bounds for thereon.

Outside RL, variants of off-policy (or counterfactual) estimation using importance sampling (or \textit{inverse propensity estimator}   \citep{horvitz1952generalization}) is common in econometrics \citep{hoover2011counterfactuals,stock2015introduction} and causal inference \cite{pearl2009causality}.
%
%
While these works have mostly focused on mean estimation, counterfactual probability density or quantiles of potential outcomes can also be estimated \citep{dinardo1995labor, melly2006estimation,chernozhukov2013inference,donald2014estimation}.
%
%
%
%
Such distribution estimation methods can possibly also be used to estimate off-policy variance; however it is unclear how to obtain unbiased estimates of the variance from an unbiased estimate of the distribution.
Instead, focusing directly on the variance can be more data-efficient and can also lead to unbiased estimators.
Further, these works neither leverage any MDP structure to reduce variance resulting from IS, nor do they provide any methods that provide high-confidence bounds on the variance.
In the RL setting, the problem of high variance in IS is exacerbated as sequential interaction leads to multiplicative importance ratios, thereby requiring additional consideration for long tails to obtain tight bounds.

\section{Experimental Study}

Inspired by real-world applications where OVE and HCOVE can be useful, we  validate our proposed estimators empirically on two domains motivated by real-world applications.
Here, we only provide a brief description about the experimental setup and the main results.
Appendix G contains additional experimental details.\\

\noindent
\textbf{Diabetes treatment: } 
This domain is based on an open-source implementation \citep{simglucose} of the FDA approved Type-1 Diabetes Mellitus simulator (T1DMS) \citep{man2014uva} for treatment of Type-1 Diabetes, where the objective is to control an insulin pump to regulate the blood-glucose level of a patient.
High-confidence estimation of the variance of a controller's outcome, before deployment, can be informative when assessing potential harm to the patient that may be caused by the controller. \\

\noindent
\textbf{Recommender system: }
This domain simulates the problem of providing online recommendations based on customer interests, where it is often useful to obtain high-confidence estimates for the variance of customer's experience, before actually deploying the system, to limit financial loss.\\

\noindent\textbf{Gridworld: }
We also consider a standard $4\times4$ Gridworld with stochastic transitions.
There are eight discrete actions corresponding to up, down, left, right, and the four diagonal movements.

%
%
%
%
Given trajectories collected using  a behavior policy $\beta$, in Figure \ref{fig:exps} we provide the trend of our estimator $\hat \sigma_n^2$ for an evaluation policy $\pi$, and the confidence intervals $\mathcal C$ and $\hat {\mathcal C}$ as the number of trajectories increase (more details on how $\pi$ and $\beta$ were constructed can be found in Appendix G).
As established in \thref{thm:unbiased} and \thref{thm:consistent}, $\hat \sigma_n^2$ can be seen to be both an unbiased and a consistent estimator of $\sigma^2(\pi)$.
Similarly, as established in \thref{thm:HCOVE1}, the $(1-\delta)$-confidence interval $\mathcal C$ provides guaranteed coverage.
In comparison, as established in \thref{thm:boot}, bootstrap bounds are approximate and can fail more than $\delta$ fraction of the time.
However, bootstrap bounds can still be useful in many applications as they provide tighter intervals.

\section{Conclusion}
In this work, we addressed an understudied problem of estimating and bounding $\sigma^2(\pi)$ using only off-policy data.
We took the first steps towards developing a model-free, off-policy, unbiased, and consistent estimator of $\sigma^2(\pi)$ using a simple double-sampling trick. We then showed how bound propagation using concentration inequalities, or statistical bootstrap, can be used to obtain CIs for $\sigma^2(\pi)$.
Finally, empirical results were provided to support the established theoretical results.

\clearpage
\section{Broader Impact}

\paragraph{Note to a wider audience:} Methods developed in this work can be beneficial for researchers and practitioners working with applications that require reliability guarantees, especially before the proposed system/policy is even deployed.
It is worth noting that if the failure rate $\delta$ is set to be too low, then our bounds can result in overly conservative intervals. 
Further, for settings where the probabilities from behavior policies are not available, and are instead estimated, $\hat \sigma_n$ might be biased.
Consequently, for applications that require hard-guarantees, or have a batch of sampled actions without their sampling probabilities, our methods are not applicable.

\paragraph{Future  research directions:}  
While we took some measures to mitigate the variance of $\hat \sigma_n$, IS can still result in high variance.
Recent off-policy \textit{mean} estimation methods show how the Markov structure of an MDP \cite{liu2018breaking, xie2019towards,rowland2020conditional} or an estimate of the model of an MDP \cite{jiang2016doubly,thomas2016data} can be further leveraged for variance reduction.
Alternatively, if the entire behavior policy is available, and not just the probabilities for the sampled actions, then multi-importance sampling \citep{sbert2018multiple} can be leveraged to relax \thref{ass:supp} and also get tighter bounds on the mean return \citep{papini2019optimistic,metelli2020importance}.
Extending these methods for OVE and HCOVE remains an interesting future direction.

%

\bibliography{mybib}

\include{appendix}

\end{document}

%% file: header.tex
\usepackage{amsmath}
\usepackage{amsthm}
\usepackage{amssymb}
\usepackage{bbm}
\usepackage{mathtools}
\usepackage{mathrsfs}
\usepackage{dsfont}
\mathtoolsset{showonlyrefs}

\newtheorem{thm}{Theorem}[]

\newtheorem{prop}{Property}[]
\newtheorem{ass}{Assumption}[]

\newtheorem{rem}{Remark}[]

\usepackage{courier} 

\usepackage{lipsum} 

\usepackage[space]{grffile} 

\usepackage{hyperref}

\usepackage{theoremref}
\usepackage{graphicx}
\usepackage{wrapfig}
\usepackage{xcolor}
\usepackage{colortbl}
\definecolor{dark-red}{rgb}{0.4,0.15,0.15}
\definecolor{dark-blue}{rgb}{0,0,0.7}

\newcommand{\topic}[1]{\textcolor{violet}{}}

\usepackage{algpseudocode}
\usepackage[vlined,linesnumbered,ruled,algo2e]{algorithm2e}
\SetKwProg{Fn}{Function}{}{}
\let\oldnl\nl
\newcommand{\nonl}{\renewcommand{\nl}{\let\nl\oldnl}}

\usepackage{multicol}
\usepackage{enumitem}
\usepackage[utf8]{inputenc} 
\usepackage[T1]{fontenc}    
\usepackage{url}            
\usepackage{booktabs}       
\usepackage{nicefrac}       

%% file: appendix.tex
\clearpage
\appendix
\onecolumn

\setcounter{lemma}{0}
\setcounter{thm}{0}
\setcounter{cor}{0}
\setcounter{prop}{0}

\section*{\centering High-Confidence Off-Policy (or Counterfactual) Variance Estimation\\ (Supplementary Material)}

\section{A: Proofs for the Na\"{i}ve Estimator $\hat \sigma_n^{2!!}$}
\label{apx:na\"{i}ve1}
\begin{prop}
    Under \thref{ass:supp}, $\hat \sigma_n^{2!!}$ may be a biased estimator of $\sigma^2(\pi)$.
    That is, it is possible that $\mathbb{E}_\beta[\hat \sigma_n^{2!!}] \neq \sigma^2(\pi)$.
\end{prop}

\begin{proof}
    We prove this using a counter-example.
    Consider the  MDP shown in Figure \ref{fig:counter}.
    \begin{figure}[h]
    \centering
    \includegraphics[width=0.2\textwidth]{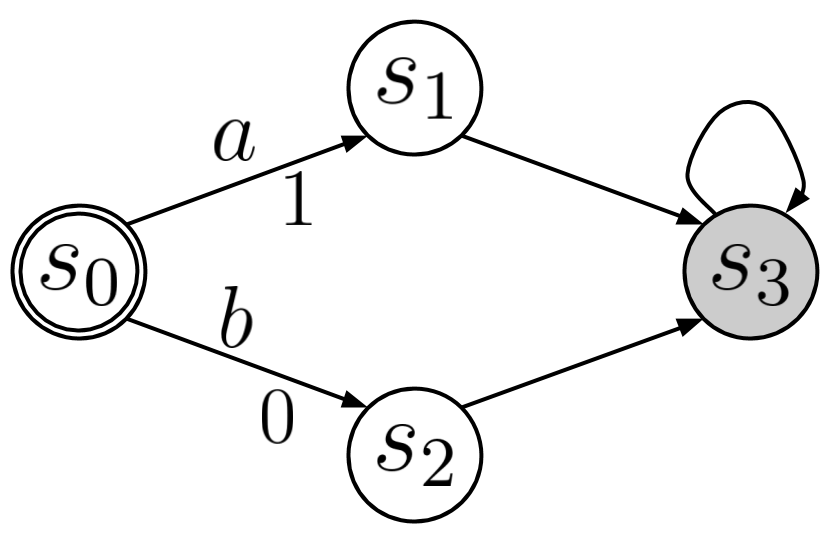}
    \caption{In this MDP, $S_0$ is the starting state and $S_3$ is the terminal/absorbing state.
    From $S_0$ there are two available actions: $a$ and $b$, which yield a reward of $1$ and $0$ respectively.
    All transitions are deterministic and  $\gamma=1$.
    }
    \label{fig:counter}
\end{figure}

    For the purpose of a counter-example, we now describe the evaluation policy $\pi$ and behavior policy $\beta$.
    Let $\pi$ be a  policy that always selects action $a$, i.e, $\pi(a|S_0) = 1$ and $\pi(b|S_0)=0$.
    Since $\pi$  is deterministic and action $a$ yields a reward of $1$ always, the variance of returns observed under $\pi$ is $0$.
    Let $\beta$ be the behavior policy which selects both actions with equal probability, i.e., $\beta(a|S_0) = \beta(b|S_0) = 0.5$

    Now, to show that $\hat \sigma_n^{2!!}$ can be a biased estimator, we explicitly compute  $\mathbb{E}[\hat \sigma_n^{2!!}]$ for the above setting, when $n=2$.
    The set of possible actions chosen by $\beta$ when $n=2$  can be $\{(a,a), (a,b), (b,a), (b,b)\}$, each of which is equally likely and occurs with   probability of $1/4$.
    Before computing variance using each of these possible outcomes, recall from \eqref{eqn:naive1}, 
    \begin{align}
        \hat \sigma_n^{2!!} &= \frac{1}{n-1}\sum_{i=1}^n \left(\rho_i G_i - \frac{1}{n} \sum_{j=1}^n \rho_j G_j \right)^2.
    \end{align}
For the case where sampled actions are $(a,a)$, the importance ratio are $(2,2)$ and
\begin{align}
    \hat \sigma_n^{2!!} = \frac{2}{2-1}\left(2\times1 - \frac{1}{2}(2\times1 + 2\times1)\right)^2 = 0.
\end{align}

Similarly, for the cases where sampled actions are $(a,b)$ or $(b,a)$, the importance ratios are $(2,0)$ or $(0,2)$ respectively, and
\begin{align}
    \hat \sigma_n^{2!!} = \frac{(2\times1 - \frac{1}{2}(2\times1 + 0))^2 + (0 - \frac{1}{2}(0 + 2\times1))^2}{2-1} = 1 + 1 = 2.
\end{align}
For the cases where sampled actions are $(b,b)$ the importance ratios are $(0,0)$ respectively, and $\hat \sigma_n^{2!!} = 0$.
Therefore, the expected value of $\hat \sigma_n^{2!!} = \frac{1}{4}(0+2+2+0) = 1 \neq 0$, and it is a biased estimator.

\end{proof}

\begin{prop}
    Under \thref{ass:supp}, $\hat \sigma_n^{2!!}$ may not be a consistent estimator of $\sigma^2(\pi)$.
    That is, it is not always the case that $\hat \sigma_n^{2!!} \overset{\text{a.s}}{\longrightarrow} \sigma^2(\pi)$.
\end{prop}

\begin{proof} We begin by expanding \eqref{eqn:naive1},
    
    \begin{align}
        \lim_{n \rightarrow \infty} \hat \sigma_n^{2!!} &= \lim_{n \rightarrow \infty} \frac{1}{n-1}\sum_{i=1}^n \left(\rho_i G_i - \frac{1}{n} \sum_{j=1}^n \rho_j G_j \right)^2
        \\
        &= \lim_{n \rightarrow \infty} \frac{1}{n-1}\sum_{i=1}^n \left(\rho_i^2 G_i^2 - \frac{2}{n} (\rho_i G_i) \sum_{j=1}^n \rho_j G_j + \left(\frac{1}{n} \sum_{j=1}^n \rho_j G_j \right)^2 \right)
        \\
        &= \lim_{n \rightarrow \infty} \left( \frac{1}{n-1}\sum_{i=1}^n \rho_i^2 G_i^2   - 2\left( \frac{1}{n-1}\sum_{i=1}^n \rho_i G_i\right) \left(\frac{1}{n}\sum_{j=1}^n \rho_j G_j \right) + \frac{1}{n-1}\sum_{i=1}^n\left(\frac{1}{n} \sum_{j=1}^n \rho_j G_j \right)^2 \right). \label{eqn:temp1}
    \end{align}
Notice, (a) $1/(n-1) = (1/n)(n/(n-1))$ and in the limit $n/(n-1) \rightarrow 1$, therefore $1/(n-1) \rightarrow 1/n$,  and (b) using Kolmogorov's strong law of large numbers, $\frac{1}{n} \sum_{j=1}^n \rho_j G_j \overset{\text{a.s.}}{\rightarrow} \mathbb{E}_\beta[\rho G] = \mathbb{E}_\pi[G]$, and $\frac{1}{n} \sum_{i=1}^n \rho_i^2 G_i^2 \overset{\text{a.s.}}{\rightarrow} \mathbb{E}_\beta[\rho^2 G^2] = \mathbb{E}_\pi[\rho G^2]$.
Therefore using Slutsky's theorem, \eqref{eqn:temp1} can be simplified to,
\begin{align}
    \lim_{n \rightarrow \infty} \hat \sigma_n^{2!!} &\overset{\text{a.s.}}{\longrightarrow}   \mathbb{E}_\pi\left[ \rho G^2 \right]   - \mathbb{E}_\pi[G]^2. \label{eqn:temp2}
\end{align}
Since $\sigma^2(\pi) = \mathbb{V}_\pi(G) =  \mathbb{E}_\pi\left[G^2 \right]   - \mathbb{E}_\pi[G]^2$, \eqref{eqn:temp2} may not be equal to $\sigma^2(\pi)$.

\end{proof}

\section{B: Proofs for the Na\"{i}ve Estimator $\hat \sigma_n^{2!}$}
\label{apx:na\"{i}ve2}




\begin{prop}
    Under \thref{ass:supp}, $\hat \sigma_n^{2!}$ may be a biased estimator of $\sigma^2(\pi)$.
    That is, it is possible that $\mathbb{E}_\beta[\hat \sigma_n^{2!}] \neq \sigma^2(\pi)$.
\end{prop}

\begin{proof}
This proof uses the same counter-example presented in the proof of \thref{prop:naive1biased}.
Recall from \eqref{eqn:naive2} that, 
    \begin{align}
        \hat \sigma_n^{2!} &= \frac{1}{n-1}\sum_{i=1}^n  \rho_i \left(G_i - \frac{1}{n} \sum_{j=1}^n \rho_j G_j \right)^2
    \end{align}
For the case where sampled actions are $(a,a)$, the importance ratios are $(2,2)$ and
\begin{align}
    \hat \sigma_n^{2!} = \frac{2}{2-1}\left(2\times \left(1 - \frac{1}{2}(2\times1 + 2\times1) \right)^2 \right) = 4.
\end{align}

Similarly, for the cases where sampled actions are $(a,b)$ or $(b,a)$, the importance ratios are $(2,0)$ or $(0,2)$ respectively, and
\begin{align}
    \hat \sigma_n^{2!} = \frac{(2\times(1 - \frac{1}{2}(2\times1 + 0))^2) + (0\times(0 - \frac{1}{2}(2\times1 + 0))^2)}{2-1} = 0 + 0 = 0.
\end{align}
For the cases where sampled actions are $(b,b)$ the importance ratios are $(0,0)$ respectively, and $\hat \sigma_n^{2!} = 0$.
Therefore, the expected value of $\hat \sigma_n^{2!} = \frac{1}{4}(4+0+0+0) = 1 \neq 0$, and it is a biased estimator.
\end{proof}

\begin{prop}
    Under \thref{ass:supp}, $\hat \sigma_n^{2!}$ is a consistent estimator of $\sigma^2(\pi)$.
    That is, $\hat \sigma_n^{2!} \overset{\text{a.s.}}{\longrightarrow} \sigma^2(\pi)$.
\end{prop}

\begin{proof} We begin by expanding \eqref{eqn:naive1},
    
    \begin{align}
        \lim_{n \rightarrow \infty} \hat \sigma_n^{2!!} &= \lim_{n \rightarrow \infty} \frac{1}{n-1}\sum_{i=1}^n\rho_i  \left( G_i - \frac{1}{n} \sum_{j=1}^n \rho_j G_j \right)^2
        \\
        &= \lim_{n \rightarrow \infty} \frac{1}{n-1}\sum_{i=1}^n \left(\rho_i G_i^2 - \frac{2}{n} (\rho_i G_i) \sum_{j=1}^n \rho_j G_j + \rho_i \left(\frac{1}{n} \sum_{j=1}^n \rho_j G_j \right)^2 \right)
        \\
        &= \lim_{n \rightarrow \infty} \left( \frac{1}{n-1}\sum_{i=1}^n \rho_i G_i^2   - 2\left( \frac{1}{n-1}\sum_{i=1}^n \rho_i G_i\right) \left(\frac{1}{n}\sum_{j=1}^n \rho_j G_j \right) + \left(\frac{1}{n-1}\sum_{i=1}^n \rho_i \right) \left(\frac{1}{n} \sum_{j=1}^n \rho_j G_j \right)^2 \right). \label{eqn:temp3}
    \end{align}
Notice, (a) $1/(n-1) = (1/n)(n/(n-1))$ and in the limit $n/(n-1) \rightarrow 1$, therefore $1/(n-1) \rightarrow 1/n$,  and (b) using Kolmogorov's strong law of large numbers, $\frac{1}{n} \sum_{j=1}^n \rho_i G_i \overset{\text{a.s.}}{\rightarrow} \mathbb{E}_\beta[\rho G] = \mathbb{E}_\pi[G]$, further $\frac{1}{n} \sum_{j=1}^n \rho_i G_i^2 \overset{\text{a.s.}}{\rightarrow} \mathbb{E}_\beta[\rho G^2] = \mathbb{E}_\pi[G^2]$, and $\frac{1}{n}\sum_{i=1}^n\rho_i \overset{\text{a.s.}}{\rightarrow} \mathbb{E}_\beta[\rho] = 1$.
Therefore using Slutsky's theorem, \eqref{eqn:temp3} can be simplified to,
\begin{align}
    \lim_{n \rightarrow \infty} \hat \sigma_n^{2!!} &\overset{\text{a.s.}}{\longrightarrow}    \mathbb{E}_\pi\left[G_i^2 \right]   - \mathbb{E}_\pi[G]^2
    \\
    &= \mathbb{V}_\pi(G) \\
    &= \sigma^2(\pi).
\end{align}
\end{proof}

\section{C: Proofs for the Proposed Estimator $\hat \sigma_n^2$}
\label{apx:proposed}
\begin{thm}
    Under \thref{ass:supp}, $\hat \sigma_n^2$ is an unbiased estimator of $\sigma^2(\pi)$.
    That is, $\mathbb{E}_\beta[\hat \sigma_n^2] = \sigma^2(\pi)$.
\end{thm}

\begin{proof}

\begin{align}
    \mathbb{E}_\beta[\hat \sigma_n^2] &= \mathbb{E}_\beta\left[\frac{1}{n}\sum_{i=1}^{n} \rho_i G_i^2  - \left(\frac{1}{|\mathcal D_1|}\sum_{i=1}^{|\mathcal D_1|}\rho_i G_i \right) \left(\frac{1}{|\mathcal D_2|}\sum_{i=1}^{|\mathcal D_2|}\rho_i G_i \right) \right]
    \\
    &= \mathbb{E}_\beta\left[\frac{1}{n}\sum_{i=1}^{n} \rho_i G_i^2 \right]  - \mathbb{E}_\beta\left[\left(\frac{1}{|\mathcal D_1|}\sum_{i=1}^{|\mathcal D_1|}\rho_i G_i \right)\right] \mathbb{E}_\beta\left[\left(\frac{1}{|\mathcal D_2|}\sum_{i=1}^{|\mathcal D_2|}\rho_i G_i \right)  \right]
    \\
    &= \mathbb{E}_\beta\left[\rho G^2 \right]  - \mathbb{E}_\beta\left[\rho G \right] \mathbb{E}_\beta\left[\rho G  \right]
    \\
    &= \mathbb{E}_\pi\left[ G^2 \right]  - \mathbb{E}_\pi\left[ G \right] \mathbb{E}_\pi\left[G  \right]
    \\
    &= \mathbb{V}_\pi [G] 
    \\
    &= \sigma^2(\pi).
\end{align}

\paragraph{Example of negative variance estimate: } For completeness, we also work out the expected estimate of $\hat \sigma_n^2$ for the counter-example used to show $\hat \sigma_n^{2!!}$ and $\hat \sigma_n^{2!}$ are biased. This example also shows that the $\hat \sigma_n$ can be negative. 
For the case where sampled actions are $(a,a)$, the importance ratio are $(2,2)$ and $\hat \sigma_n$ is negative,
\begin{align}
    \hat \sigma_n^2 = \frac{1}{2}(2\times 1 + 2\times 1) - (2 \times1) (2\times1) = -2.
\end{align}
Similarly, for the cases where sampled actions are $(a,b)$ or $(b,a)$, the importance ratios are $(2,0)$ or $(0,2)$ respectively, and
\begin{align}
    \hat \sigma_n^2 = \frac{1}{2}(2\times 1 + 0) - (2\times1)\times (0) = 1.
\end{align}
For the cases where sampled actions are $(b,b)$, the importance ratios are $(0,0)$ respectively, and $\hat \sigma_n^2 = 0$.
Therefore, the expected value of $\hat \sigma_n^2 = \frac{1}{4}(-2+1+1+0) = 0$, as required.

\end{proof}

\begin{thm}
    Under \thref{ass:supp}, $\hat \sigma_n^2$ is a consistent estimator of $\sigma^2(\pi)$.
    That is, $\hat \sigma_n^2 \overset{\text{a.s.}}{\longrightarrow} \sigma^2(\pi)$.
\end{thm}

\begin{proof}

\begin{align}
    \lim_{n \rightarrow \infty}  \hat \sigma_n^2 &= \lim_{n \rightarrow \infty}  \left[\frac{1}{n}\sum_{i=1}^{n} \rho_i G_i^2  - \left(\frac{1}{|\mathcal D_1|}\sum_{i=1}^{|\mathcal D_1|}\rho_i G_i \right) \left(\frac{1}{|\mathcal D_2|}\sum_{i=1}^{|\mathcal D_2|}\rho_i G_i \right) \right]. \label{eqn:temp4}
\end{align}
Now using Kolmogorov's strong law of large numbers and Slutsky's theorem, \eqref{eqn:temp4} can be simplified to
\begin{align}
    \lim_{n \rightarrow \infty} &= \left[\lim_{n \rightarrow \infty}  \frac{1}{n}\sum_{i=1}^{n} \rho_i G_i^2 \right]  - \left[\lim_{n \rightarrow \infty} \left(\frac{1}{|\mathcal D_1|}\sum_{i=1}^{|\mathcal D_1|}\rho_i G_i \right)\right] \left[\lim_{n \rightarrow \infty} \left(\frac{1}{|\mathcal D_2|}\sum_{i=1}^{|\mathcal D_2|}\rho_i G_i \right)  \right]
    \\
    &\overset{\text{a.s.}}{\longrightarrow} \mathbb{E}_\beta\left[\rho G^2 \right]  - \mathbb{E}_\beta\left[\rho G \right] \mathbb{E}_\beta\left[\rho G  \right]
    \\
    &= \mathbb{E}_\pi\left[ G^2 \right]  - \mathbb{E}_\pi\left[ G \right] \mathbb{E}_\pi\left[G  \right]
    \\
    &= \mathbb{V}_\pi [G] 
    \\
    &= \sigma^2(\pi).
\end{align}

\end{proof}

\begin{thm}
Under \thref{ass:supp}, 
\begin{align}
    \mathbb{E}_\beta\left[\rho G^2 \right] = \mathbb{E}_\beta\left[\sum_{i=0}^{T} \sum_{j=0}^T \rho\left(0,\texttt{max}(i,j)\right) \gamma^{i+j} R_{(i)} R_{(j)}\right].
\end{align}
\end{thm}

\begin{proof}

\begin{align}
    \mathbb{E}_\beta[\rho G^2] &= \mathbb{E}_\beta\left[\rho(0,T) \left(\sum_{i=0}^T \gamma^{i} R_{(i)} \right)^2 \right]
    \\
    &= \mathbb{E}_\beta\left[\rho(0,T) \left(\sum_{i=0}^T \gamma^{2i} R_{(i)}^2 + 2 \sum_{i=0}^{T} \sum_{j=i+1}^T \gamma^{i+j} R_{(i)} R_{(j)} \right) \right]
    \\
    &= \left( \sum_{i=0}^T \mathbb{E}_\beta\left[\rho(0,T)\gamma^{2i} R_i^2\right] \right)  + \left( 2 \sum_{i=0}^{T} \sum_{j=i+1}^T \mathbb{E}_\beta\left[\rho(0,T) \gamma^{i+j} R_{(i)} R_{(j)}  \right]\right)
    \\
    &=  \sum_{i=0}^T \mathbb{E}_\beta\left[\rho(0,T)\gamma^{2i} R_{(i)}^2\right]  + 2 \sum_{i=0}^{T} \sum_{j=i+1}^T \mathbb{E}_\beta\left[\rho(0,T) \gamma^{i+j} R_{(i)} R_{(j)}  \right]
    \\
    &= \left( \sum_{i=0}^T \sum_{h \in \mathcal H_{0:T}^\beta} \Pr(h|\beta)\rho(0,T)\gamma^{2i} R_{(i)}^2 \right) + \left( 2 \sum_{i=0}^{T} \sum_{j=i+1}^T \sum_{h \in \mathcal H_{0:T}^\beta} \Pr(h|\beta)\rho(0,T) \gamma^{i+j} R_{(i)} R_{(j)} \right)
    \\
    &\overset{(a)}{=}  \left(\sum_{i=0}^T \sum_{h \in \mathcal H_{0:T}^\pi} \Pr(h|\pi)\gamma^{2i} R_{(i)}^2 \right) + \left(2 \sum_{i=0}^{T} \sum_{j=i+1}^T \sum_{h \in \mathcal H_{0:T}^\pi} \Pr(h|\pi) \gamma^{i+j} R_{(i)} R_{(j)} \right)
    \\
    &= \left( \sum_{i=0}^T \sum_{h \in \mathcal H_{0:i}^\pi} \Pr(h|\pi)\gamma^{2i} R_{(i)}^2 \sum_{h' \in \mathcal H_{i+1:t}^\pi} \Pr(h'|h,\pi) \right) + \left(2 \sum_{i=0}^{T} \sum_{j=i+1}^T \sum_{h \in \mathcal H_{0:j}^\pi} \Pr(h|\pi) \gamma^{i+j} R_{(i)} R_{(j)} \sum_{h' \in \mathcal H_{j+1:t}^\pi} \Pr(h'|h,\pi) \right)
    \\
    &= \left( \sum_{i=0}^T \sum_{h \in \mathcal H_{0:i}^\pi} \Pr(h|\pi)\gamma^{2i} R_{(i)}^2  \right) + \left(2 \sum_{i=0}^{T} \sum_{j=i+1}^T \sum_{h \in \mathcal H_{0:j}^\pi} \Pr(h|\pi) \gamma^{i+j} R_{(i)} R_{(j)} \right)
    \\
    &\overset{(b)}{=} \left( \sum_{i=0}^T \sum_{h \in \mathcal H_{0:i}^\beta} \Pr(h|\beta) \rho(0, i)\gamma^{2i} R_{(i)}^2  \right) + \left(2 \sum_{i=0}^{T} \sum_{j=i+1}^T \sum_{h \in \mathcal H_{0:j}^\beta} \Pr(h|\beta)\rho(0,j) \gamma^{i+j} R_{(i)} R_{(j)}  \right)
    \\
    &= \left( \sum_{i=0}^T \mathbb{E}_\beta \left[ \rho(0, i)\gamma^{2i} R_{(i)}^2  \right] \right) + \left(2 \sum_{i=0}^{T} \sum_{j=i+1}^T \mathbb{E}_\beta \left[\rho(0,j) \gamma^{i+j} R_{(i)} R_{(j)}  \right]\right)
    \\
    &=\mathbb{E}_\beta \left[ \sum_{i=0}^T  \rho(0, i)\gamma^{2i} R_{(i)}^2  + 2 \sum_{i=0}^{T} \sum_{j=i+1}^T \rho(0,j) \gamma^{i+j} R_{(i)} R_{(j)}  \right]
    \\
    &= \mathbb{E}_\beta\left[\sum_{i=0}^{T} \sum_{j=0}^T \rho\left(0,\texttt{max}(i,j)\right) \gamma^{i+j} R_{(i)} R_{(j)}\right],
\end{align}
where steps (a) and (b) follow due to \thref{ass:supp}.
\end{proof}

\section{Proofs for HCOVE using Concentration Inequalities}

\begin{thm} Under \thref{ass:supp}, if $(\delta_1 + \delta_2 + \delta_3 + \delta_4) \leq \delta$, then for the confidence interval $\mathcal C \coloneqq [v^\texttt{lb}, v^\texttt{ub}]$,
\begin{align}
    \Pr \left(\sigma^2(\pi) \in \mathcal C\right) \geq 1 - \delta.
\end{align}
\end{thm}

\begin{proof}
For brevity, let $X \coloneqq \mathbb{E}_\beta[\rho G^2], \text{and}\,\, Y \coloneqq \mathbb{E}_\beta[\rho G]^2$.

\begin{align}
    \Pr(v^\texttt{lb} \leq \sigma^2(\pi)) 
    &= \Pr \left(\operatorname{CI}_-(X, \delta_1) - \operatorname{CI}^+(Y^2, \delta_2) \leq X - Y \right)
    \\
    &\geq \Pr\left((\operatorname{CI}_-(X, \delta_1) \leq X) \cap (Y \leq \operatorname{CI}^+(Y^2, \delta_2)  ) \right)
    \\
    &= 1 -  \Pr\left(((\operatorname{CI}_-(X, \delta_1)  \leq X) \cap (Y  \leq  \operatorname{CI}^+(Y^2, \delta_2) ))^c\right)
    \\
    &= 1 -  \Pr\left((\operatorname{CI}_-(X, \delta_1)  \leq X)^c \cup (Y  \leq  \operatorname{CI}^+(Y^2, \delta_2) )^c \right)
    \\
    &= 1 -  \Pr\left((\operatorname{CI}_-(X, \delta_1)  > X) \cup (Y  >  \operatorname{CI}^+(Y^2, \delta_2) ) \right)
    \\
    &= 1 - \left(\delta_1 + \delta_2 \right),
\end{align}
where superscript of $c$ represents complement.
Similarly, it can be shown that $\Pr(v^\texttt{ub}\geq \sigma^2(\pi)) \geq 1 - (\delta_3 + \delta_4)$.
Therefore, the maximum probability of failure, i.e., either $v^\texttt{ub}< \sigma^2(\pi)$ or $v^\texttt{lb} > \sigma^2(\pi)$, is less than $(\delta_1 + \delta_2 + \delta_3 + \delta_4)$, which is not greater than $\delta$.
\end{proof}

\begin{thm}
    Let $X$ be either $G$ or $G^2$, then for any $\delta \in [0, 0.5]$ and a fixed constant $\xi$,
    \begin{align}
         \operatorname{CI}^+_-(\mathbb{E}_\beta[\rho X], \delta) = \operatorname{CI}^+_-(\mathbb{E}_\beta[\rho (X - \xi)], \delta) + \xi. 
    \end{align}
\end{thm}

\begin{proof}

\begin{align}
    \operatorname{CI}^+_-(\mathbb{E}_\beta[\rho X], \delta) &= \operatorname{CI}^+_-(\mathbb{E}_\beta[\rho X - \rho \xi + \rho \xi], \delta)
    \\
     &\overset{(a)}{=} \operatorname{CI}^+_-(\mathbb{E}_\beta[\rho (X - \xi)] + \mathbb{E}_\beta[\rho] \xi, \delta)
     \\
     &\overset{(b)}{=} \operatorname{CI}^+_-(\mathbb{E}_\beta[\rho (X - \xi)] + \xi, \delta)
     \\
     &\overset{(c)}{=} \operatorname{CI}^+_-(\mathbb{E}_\beta[\rho (X - \xi)], \delta) + \xi,
\end{align}
where (a) and (c) follow because $\xi$ is a fixed constant, and $(b)$ follows because $\mathbb{E}_\beta[\rho] = 1$.
\end{proof}

\begin{thm} Under \thref{ass:supp}, for any $\delta \in [0. 0.5]$, let $\xi_R \coloneqq \texttt{max}(R_\texttt{min}^2, R_\texttt{max}^2)$ and $\xi_G \coloneqq \texttt{max}(G_\texttt{min}^2, G_\texttt{max}^2)$ then
\begin{align}
    X \coloneqq& \sum_{i=0}^{T} \sum_{j=0}^T \rho\left(0,\texttt{max}(i,j)\right) \gamma^{i+j} \left( R_{(i)} R_{(j)} - \xi_R \right),
    \\
    Y \coloneqq& \sum_{i=0}^T  \rho(0, i)\gamma^i \left( R_{(i)} - R_\texttt{max}\right),
\end{align}
then $\Pr(X \leq 0) = \Pr(Y \leq 0) =1$, and
\begin{align}
    \operatorname{CI}^+\left(\mathbb{E}_\beta[\rho G^2], \delta\right) &= \operatorname{CI}^+\left(\mathbb{E}_\beta[X], \delta\right) + \xi_G,
    \\
    \operatorname{CI}^+\left(\mathbb{E}_\beta[\rho G], \delta\right) &= \operatorname{CI}^+\left(\mathbb{E}_\beta[Y], \delta\right) + G_\texttt{max}.
\end{align}
\end{thm}

\begin{proof}
Let $c \coloneqq (1 -\gamma^T)/(1- \gamma)$.
Then $\mathbb{E}_\beta[x]$ is.
\begin{align}
    \mathbb{E}_\beta[X] &= \mathbb{E}_\beta \left[ \sum_{i=0}^{T} \sum_{j=0}^T \rho\left(0,\texttt{max}(i,j)\right) \gamma^{i+j} \left( R_{(i)} R_{(j)} \right) \right] - \mathbb{E}_\beta \left[ \sum_{i=0}^{T} \sum_{j=0}^T \rho\left(0,\texttt{max}(i,j)\right) \gamma^{i+j} \left(\xi_R \right) \right] 
    \\
     &\overset{(a)}{=} \mathbb{E}_\beta \left[\rho G^2 \right] - \xi_R  \sum_{i=0}^{T} \sum_{j=0}^T \gamma^{i+j} \mathbb{E}_\beta \left[\rho\left(0,\texttt{max}(i,j)\right)   \right]
     \\
     &\overset{(b)}{=} \mathbb{E}_\beta \left[\rho G^2 \right] - \xi_R  \sum_{i=0}^{T} \gamma^i \sum_{j=0}^T \gamma^{j}  
     \\
     &= \mathbb{E}_\beta \left[\rho G^2 \right] - c^2 \xi_R, \label{eqn:qw12}  
\end{align}
where (a) follows from \thref{thm:cdis} and (b) follows because $\mathbb{E}_\beta \left[\rho\left(0,\texttt{max}(i,j)\right)\right] = 1$.
Notice that as $c^2\xi_R$ is equivalent to $\xi_G$, therefore substituting it into \eqref{eqn:qw12} gives,
\begin{align}
    \operatorname{CI}^+\left(\mathbb{E}_\beta[X], \delta\right) + \xi_G &= \operatorname{CI}^+\left(\mathbb{E}_\beta[\rho G^2] - \xi_G, \delta\right) + \xi_G,
    \\
    &= \operatorname{CI}^+\left(\mathbb{E}_\beta[\rho G^2], \delta\right).
\end{align}
Similarly, 
\begin{align}
     \mathbb{E}_\beta[Y] &= \mathbb{E}_\beta\left[\sum_{i=0}^T  \rho(0, i)\gamma^i  R_{(i)}  \right] - \mathbb{E}_\beta\left[\sum_{i=0}^T  \rho(0, i)\gamma^i  R_\texttt{max} \right]
     \\
     &= \mathbb{E}_\beta\left[\rho G \right] - R_\texttt{max} \sum_{i=0}^T \gamma^i  \mathbb{E}_\beta\left[\rho(0, i) \right]
     \\
     &\overset{(c)}{=} \mathbb{E}_\beta\left[\rho G \right] - cR_\texttt{max},
\end{align}
where (c) follows because $\mathbb{E}_\beta\left[\rho(0, i) \right] = 1$. Further, as $cR_\texttt{max} = G_\texttt{max}$, 
\begin{align}
     \operatorname{CI}^+\left(\mathbb{E}_\beta[Y], \delta\right) + G_\texttt{max} &= \operatorname{CI}^+\left(\mathbb{E}_\beta\left[\rho G \right] - G_\texttt{max}, \delta\right) + G_\texttt{max}.
     \\
     &=  \operatorname{CI}^+\left(\mathbb{E}_\beta\left[\rho G \right], \delta\right).
\end{align}
\end{proof}

\section{E. Proofs for HCOVE using Statistical Bootstrapping}
\begin{thm} Under \thref{ass:supp,ass:bounded}, the confidence interval $\hat {\mathcal C}$ has a finite sample error of $O(n^{-1/2})$. That is,
\begin{align}
    \Pr \left(\sigma^2(\pi) \in \hat {\mathcal C}\right) \geq 1 - \delta - O\left(n^{-1/2}\right).
\end{align}
\end{thm}

\begin{proof}
    This proof directly leverages the finite-sample coverage error result by \citet{efron1994introduction}.
    A similar technique has been used by \citet{kostrikov2020statistical} for establishing the finite-sample coverage error of the CIs for the the mean return $\mu(\pi)$ using off-policy data.
    Our result is inspired by theirs and establishes  finite-sample coverage error of the CIs for the variance of returns, $\sigma^2(\pi)$.
    Before proceeding, we first define some additional notation and then review \textit{Hadamard differentiability} \cite{wasserman2006all}, which is a key property for establishing the validity of bootstrap.

    For brevity, let $\mathcal H \coloneqq \mathcal H_{(0):(T)}$.
    For a trajectory $x \in \mathcal H$, let $\rho(x)$ be the importance ratio of the entire trajectory, $g_1(x) \coloneqq G$ be the return, and $g_2(x) \coloneqq G^2$ be the return squared.
    Considering a finite set of possible trajectories $\mathcal H$, for a given set of trajectories $\mathcal D$, let the empirical distribution over the trajectories be,
    \begin{align}
        d^{\mathcal D}(x) \coloneqq \frac{1}{|\mathcal D|} \sum_{h \in \mathcal D} \delta_{\{h=x\}}.
    \end{align}

    \noindent
    \textbf{Hadamard Differentiability: } Suppose $F$ is a functional mapping distributions $\mathcal P$  over trajectories to $\mathbb R$.
    Denote $\mathcal P_L$ as the linear space generated by $\mathcal P$.
    The functional $F$ is said to be Hadamard differentiable at $d^{\mathcal D} \in \mathcal P$ if there exists a linear functional $L_{\mathcal D}$ on $\mathcal P_L$ such that for any $\epsilon_n \rightarrow 0$ and $\{ P,P_1,P_2,P_3,...\} \subset \mathcal P_L$ such that $\lVert P_n - P\rVert_\infty\rightarrow 0$ and $d^{\mathcal D} + \epsilon_n P_n \in \mathcal P$,
    \begin{align}
        \lim_{n\rightarrow \infty} \left\lvert  \frac{F(d^{\mathcal D} + \epsilon_n P_n) - F(d^{\mathcal D})}{\epsilon_n} - L_{\mathcal D}(P) \right\rvert = 0.
    \end{align}
    
    In the following, we directly leverage the finite sample coverage error rate established for bootstrap \cite{efron1994introduction} by considering the functional $F$ to be our estimator $\hat \sigma_n^2$ and showing that it is Hadamard differentiable for all $d^{\mathcal D}$.
    To make the dependence of $d^{\mathcal D}$ explicit, we write $\hat \sigma_n^2(d^{\mathcal D})$ instead of $\hat \sigma_n^2$.
    Now using \eqref{eqn:proposed},
    \begin{align}
        \hat \sigma_n^2(d^{\mathcal D} + \epsilon_n P_n) &= \sum_{x \in \mathcal H}(d^{\mathcal D} + \epsilon_n P_n)(x)\rho(x) g_2(x) - \left(\sum_{x \in \mathcal H}(d^{\mathcal D_1} + \epsilon_n P_n)(x) \rho(x) g_1(x) \right)\left(\sum_{y \in \mathcal H}(d^{\mathcal D_2} + \epsilon_n P_n)(y) \rho(y) g_1(y) \right) , \label{eqn:bottperb}
        \\
        \hat \sigma_n^2(d^{\mathcal D}) &= \sum_{x \in \mathcal H}d^{\mathcal D}(x)\rho(x) g_2(x) - \left(\sum_{x \in \mathcal H}d^{\mathcal D_1}(x) \rho(x) g_1(x) \right)\left(\sum_{y \in \mathcal H}d^{\mathcal D_2}(y) \rho(y) g_1(y) \right) . \label{eqn:bootnoperb}
    \end{align}
    Using \eqref{eqn:bottperb} and \eqref{eqn:bootnoperb}, as $n \rightarrow \infty$ then $\epsilon_n \rightarrow 0$,
    \begin{align}
        \lim_{\epsilon_n \rightarrow 0}\frac{\hat \sigma_n^2(d^{\mathcal D} + \epsilon_n P_n)  - \hat \sigma_n^2(d^{\mathcal D}) }{\epsilon_n} =& \sum_{x \in \mathcal H} P_n(x)\rho(x) g_2(x)
        \\ &- \sum_{x \in \mathcal H}\sum_{y \in \mathcal H} \rho(x) g_1(x) \rho(y) g_1(y) \left[ d^{\mathcal D_1}(x)P_n(y) + d^{\mathcal D_2}(y)P_n(x) + \underbrace{\epsilon_n P_n(x)P_n(y)}_{=0 \text{ as  } \epsilon_n \rightarrow 0}   \right]. \label{eqn:bootlast}
    \end{align}
    It can be seen that \eqref{eqn:bootlast} is linear in $P_n$, so there exists a linear functional $L_{\mathcal D}$ on $P_L$ such that $\hat \sigma_n^2(d^{\mathcal D})$ is Hadamard differnetiable. 
\end{proof}

\section{F. Algorithms}

In this section we present the algorithms to obtain high-confidence bounds for $\sigma^2(\pi)$.
Algorithms \ref{alg:hcoveciu}--\ref{alg:hcovecil} provide lower and upper bounds using concentration inequalities.
Algorithm \ref{alg:hcoveboot} provides lower and upper bounds using statistical bootstrapping.
In the following, we briefly review the concentration inequality established by \citet{thomas2015high}, which we also use in Algorithms \ref{alg:hcoveciu}--\ref{alg:hcovecil}.

\begin{thm}[\cite{thomas2015high}]
Let $\{X_i\}_{i=1}^n$  be $n$ independent real-valued  bounded  random  variables  such  that  for  each $i \in \{1,...,n\}$, we  have $\Pr(0 \leq X_i)  =  1$ ,$\mathbb{E}[X_i] \leq \mu$,  and the  fixed  real-valued  threshold $c_i > 0$.  Let $\delta >0$ and $Y_i \coloneqq  \texttt{min}(X_i,c_i)$. 
\begin{align}
  \operatorname{CI}^-\left(\{X_i\}_{i=1}^n, \delta \right) \coloneqq \left( \sum_{i=1}^n \frac{1}{|c_i|} \right)^{-1}\!\!\! \sum_{i=1}^n \frac{Y_i}{|c_i|} - \left( \sum_{i=1}^n \frac{1}{|c_i|} \right)^{-1} \!\! \frac{7n \ln(2/\delta)}{3(n-1)} - \left( \sum_{i=1}^n \frac{1}{|c_i|} \right)^{-1} \!\!\!\sqrt{ \frac{\ln(2/\delta)}{n-1} \sum_{i,j=1}^n \left(\frac{Y_i}{|c_i|} - \frac{Y_j}{|c_j|} \right)^2   }. 
  \label{eqn:cilower}
\end{align}
Then with probability at least $1 - \delta$, we have $\mu \geq  \operatorname{CI}^-\left(\{X_i\}_{i=1}^n, \delta \right)$.
\thlabel{thm:thomas}
\end{thm}

Similarly, let $\{A_i\}_{i=1}^n$  be $n$ independent real-valued  bounded  random  variables where $\Pr(0 \geq A_i)  =  1$ and $\mathbb{E}[A_i] \geq \mu$,  then for a fixed  real-valued  threshold $c_i < 0$ and $B_i \coloneqq  \texttt{max}(A_i,c_i)$, the expected value $\mathbb{E}[B_i] \geq \mathbb{E}[A_i]$. Therefore, an upper bound on $\mathbb{E}[B_i]$ is also an upper bound on $\mathbb{E}[A_i]$. Consequently, to get an upper bound on $\mathbb{E}[B_i]$ we flip the bound in \thref{thm:thomas} (i.e., let $Y_i \coloneqq - B_i$ in \eqref{eqn:cilower} and then negate the resulting bound, since $\forall \nu, \, \Pr(\mathbb{E}[Y_i] \geq \nu) = \Pr(-\mathbb{E}[B_i] \geq \nu) =  \Pr(\mathbb{E}[B_i] \leq -\nu)$). 
%

\begin{align}
  \operatorname{CI}^+\left(\{A_i\}_{i=1}^n, \delta \right) \coloneqq \left( \sum_{i=1}^n \frac{1}{|c_i|} \right)^{-1} \!\!\! \sum_{i=1}^n \frac{B_i}{|c_i|} + \left( \sum_{i=1}^n \frac{1}{|c_i|} \right)^{-1} \!\! \frac{7n \ln(2/\delta)}{3(n-1)} + \left( \sum_{i=1}^n \frac{1}{|c_i|} \right)^{-1} \!\!\!\sqrt{ \frac{\ln(2/\delta)}{n-1} \sum_{i,j=1}^n \left(\frac{B_i}{|c_i|} - \frac{B_j}{|c_j|} \right)^2   }. \label{eqn:ciupper}
\end{align}
Then with probability at least $1 - \delta$, we have $\mu \leq  \operatorname{CI}^+\left(\{A_i\}_{i=1}^n, \delta \right)$.
In \eqref{eqn:cilower}, $c_i$'s help in truncating the upper tail of the distribution, and in \eqref{eqn:ciupper}, $c_i$'s help in truncating the lower tail of the distribution.
Further, note the use of absolute values for $c_i$'s in our presentation of the bounds by \citet{thomas2015high}; while in \eqref{eqn:ciupper} this is redundant as $c_i > 0$, in \eqref{eqn:cilower} this is important to prevent the change in sign of the random variable when normalized using $c_i$'s as in this equation $c_i < 0$. 
%
%
%
For simplicity, \citet{thomas2015high} suggest setting a common $c$ for all $c_i$'s.
Further, since the value of $c$ should be chosen independent of that data being analyzed, they suggest partitioning the data into two sets $\mathcal D_\text{pre}$ and $\mathcal D_\text{post}$ in the ratio $1/20 :19/20$ and searching the value of $c$ that optimizes the bound on the data from $\mathcal D_\text{pre}$.
The value of this $c$ is then used to get the desired bounds using data from $\mathcal D_\text{post}$.
We refer the readers to the work by \citet{thomas2015high} for more details.

	\IncMargin{1em}
	\begin{algorithm2e}[t]
		\textbf{Input:} {Dataset $\mathcal D$}, Confidence level $1 - \delta$  
		\\
	    \vspace{5pt}
	    \nonl \textcolor[rgb]{0.5,0.5,0.5}{\# Upper Bound on $\mathbb{E}_\beta[\rho G^2]$}
	    \\
		$\xi_R = \texttt{max}(R_\texttt{min}^2, R_\texttt{max}^2)$ and $\xi_G = \texttt{max}(G_\texttt{min}^2, G_\texttt{max}^2)$
        \\
        $X_i = \sum_{j=0}^{T} \sum_{k=0}^T \rho_i\left(0,\texttt{max}(j,k)\right) \gamma^{j+k} \left( R_{i(j)} R_{i(k)} - \xi_R \right)$
        \\
        $X^\texttt{ub} = \operatorname{CI}^+\left(\{X_i\}_{i=1}^{|\mathcal D|}, \delta/2 \right) + \xi_G$
        \\
        
	    \vspace{5pt}
	    \nonl \textcolor[rgb]{0.5,0.5,0.5}{\# Upper and Lower Bounds on $\mathbb{E}_\beta[\rho G]$}
	    \\
        $Y^-_i = \sum_{j=0}^T  \rho_i(0, j)\gamma^j \left( R_{i(j)} - R_\texttt{min}\right)$,
        \\
        $Y^+_i = \sum_{j=0}^T  \rho_i(0, j)\gamma^j \left( R_{i(j)} - R_\texttt{max}\right)$,
        \\
        $Y^\texttt{lb} = \operatorname{CI}^-\left(\{Y^-_i\}_{i=1}^{|\mathcal D|}, \delta/4\right) + G_\texttt{min}$.
        \\
        $Y^\texttt{ub} = \operatorname{CI}^+\left(\{Y^+_i\}_{i=1}^{|\mathcal D|}, \delta/4\right) + G_\texttt{max}$.
        \\
        
	    \vspace{5pt}
	    \nonl \textcolor[rgb]{0.5,0.5,0.5}{\# Lower Bound on $\mathbb{E}_\beta[\rho G]^2$ using bound propagation}
	    \\
        \eIf{$Y^\texttt{lb} \leq 0 \leq Y^\texttt{ub} $}{$Z^\texttt{lb}  = 0$}{$Z^\texttt{lb} = \texttt{min}((Y^\texttt{lb})^2, (Y^\texttt{ub} )^2)$}
		\textbf{Return} $X^\texttt{ub} - Z^\texttt{lb} $
		\caption{HCOVE - upper bound (Concentration Inequality)}
	    \label{alg:hcoveciu}
	\end{algorithm2e}
	\DecMargin{1em}

	\IncMargin{1em}
	\begin{algorithm2e}[t]
		\textbf{Input:} {Dataset $\mathcal D$}, Confidence level $1 - \delta$  
		\\
	    \vspace{5pt}
	    \nonl \textcolor[rgb]{0.5,0.5,0.5}{\# Lower Bound on $\mathbb{E}_\beta[\rho G^2]$}
	    \\ 
	    \nonl \textcolor[rgb]{0.5,0.5,0.5}{(Control variate is not needed as $\rho G^2$ is always positive.)}
	    \\
		%
        %
        $X_i = \sum_{j=0}^{T} \sum_{k=0}^T \rho_i\left(0,\texttt{max}(j,k)\right) \gamma^{j+k} \left( R_{i(j)} R_{i(k)} \right)$
        \\
        $X^\texttt{lb} = \operatorname{CI}^-\left(\{X_i\}_{i=1}^{|\mathcal D|}, \delta/2 \right) $
        \\
        
	    \vspace{5pt}
	    \nonl \textcolor[rgb]{0.5,0.5,0.5}{\# Upper and Lower Bounds on $\mathbb{E}_\beta[\rho G]$}
	    \\
        $Y^-_i = \sum_{j=0}^T  \rho_i(0, j)\gamma^j \left( R_{i(j)} - R_\texttt{min}\right)$,
        \\
        $Y^+_i = \sum_{j=0}^T  \rho_i(0, j)\gamma^j \left( R_{i(j)} - R_\texttt{max}\right)$,
        \\
        $Y^\texttt{lb} = \operatorname{CI}^-\left(\{Y^-_i\}_{i=1}^{|\mathcal D|}, \delta/4\right) + G_\texttt{min}$.
        \\
        $Y^\texttt{ub} = \operatorname{CI}^+\left(\{Y^+_i\}_{i=1}^{|\mathcal D|}, \delta/4\right) + G_\texttt{max}$.
        \\
        
	    \vspace{5pt}
	    \nonl \textcolor[rgb]{0.5,0.5,0.5}{\# Upper Bound on $\mathbb{E}_\beta[\rho G]^2$ using bound propagation}
	    \\
	    $Z^\texttt{ub} = \texttt{max}((Y^\texttt{lb})^2, (Y^\texttt{ub} )^2)$
		\\
		\textbf{Return} $X^\texttt{lb} - Z^\texttt{ub} $
		\caption{HCOVE - lower bound (Concentration Inequality)}
	    \label{alg:hcovecil}
	\end{algorithm2e}
	\DecMargin{1em}

	\IncMargin{1em}
	\begin{algorithm2e}[t]
		\textbf{Input:} {Dataset $\mathcal D$}, Confidence level $1 - \delta$ \\
		Compute variance estimate $\hat \sigma_n^2 =$ \textbf{Algorithm 1}($\mathcal D$)
		\\
		Compute $B$ bootstrapped datasets $\{\mathcal D_i^*\}_{i=1}^B$
		\\
		Compute bootstrapped variance estimates $\{\hat \sigma^{2*}_{n,i}\}_{i=1}^B$ using \textbf{Algorithm 1} for $\{\mathcal D_i^*\}_{i=1}^B$.
		\\
		Compute $\delta/2$ and $1 - \delta/2$ quantiles $z_{\delta/2}, z_{1-\delta/2}$ of $\{\hat \sigma^{2*}_{n,i} - \hat \sigma_n^2\}_{i=1}^B$
		\\
		\textbf{Return} $\hat {\mathcal C} \coloneqq [\hat \sigma_n^2 - z_{1-\delta/2}, \hat \sigma_n^2 - z_{\delta/2}]$
		\caption{HCOVE - upper and lower bounds (Bootstrap)}
	\label{alg:hcoveboot}
	\end{algorithm2e}
	\DecMargin{1em}

\section{G. Empirical Details}

In this section, we discuss domain details and how $\pi$ and $\beta$ were selected for both the domains.

\paragraph{Recommender System: }
Online recommendation systems are popular for  tutorials, movies, advertisements, etc.
In all these settings it may be beneficial to assess the variability in the customer's experience once the new system/policy is deployed.
To abstract such settings, we create a simulated domain where the interest of the user for a finite set of items is represented using the reward for the corresponding item.

We use an actor-critic algorithm \cite{SuttonBarto2} to find a near optimal policy policy $\pi$, which we use as the evaluation policy.
Let $\pi^\texttt{rand}$ be a random policy with uniform distribution over the actions (items).
Then for an $\alpha = 0.5$, we define the behavior policy $\beta(a|s) \coloneqq \alpha \pi(a|s) + (1-\alpha) \pi^\texttt{rand}(a|s)$ for all states and actions.
\paragraph{Gridworld: }
We also consider a standard $4\times4$ Gridworld with stochastic transitions.
There are eight discrete actions corresponding to up, down, left, right, and the four diagonal movements.
Behavior and the evaluation policy for this domain were obtained in a similar way as discussed for the recommender system domain.

\paragraph{Diabetes Treatment: } 
This domain is modeled using an open-source implementation \citep{simglucose} of the U.S. Food and Drug Administration (FDA) approved Type-1 Diabetes Mellitus simulator (T1DMS) \citep{man2014uva} for the treatment of Type-1 diabetes.
An episode corresponds to a day, and each step of an episode corresponds to a minute in an \textit{in-silico} patient's body and is governed by a continuous time non-linear ordinary differential equation (ODE) \citep{man2014uva}.

To control the insulin injection, which is required for regulating the blood glucose level, we use a parameterized policy based on the amount of insulin that a person with diabetes is instructed to inject prior to eating a meal \citep{bastani2014model}:
\begin{align}
    \text{injection} = \frac{\text{current blood glucose} - \text{target blood glucose}}{CF} + \frac{\text{meal size}}{CR},
\end{align}
where `current blood glucose' is the estimate of the person's current blood glucose level, `target blood glucose' is the desired blood glucose, `meal size' is the estimate of the size of the meal the patient is about to eat, and $CR \in [CR_\texttt{min}, CR_\texttt{max}]$ and $CF \in [CF_\texttt{min}, CF_\texttt{max}]$ are two real-valued parameters that must be tuned based on the body parameters to make the treatment effective.

The action distribution for the policy is parameterized using a normal distribution $\mathcal N(\mu, \sigma)$, whose mean $\mu$ is obtained using a sigmoid function (scaled for the desired range), and the standard deviation $\sigma$ is kept fixed.
We use an actor-critic algorithm \cite{SuttonBarto2} to find a near optimal policy $\pi$ having normal distribution $\mathcal N(\mu_\pi, \sigma)$ , which we use as the evaluation policy.
Let $\pi^\texttt{rand}$ be a random policy parameterized using a normal distribution $\mathcal N(\mu_\texttt{rand}, \sigma)$, where $\mu_\texttt{rand} \coloneqq [ (CR_\texttt{max} - CR_\texttt{min})/2, (CF_\texttt{max} - CF_\texttt{min})/2] $.
Then for an $\alpha = 0.5$, we parameterize the behavior policy $\beta$ using a normal distirbution $\mathcal N(\mu_\beta, \sigma)$, where $\mu_\beta \coloneqq \alpha \mu_\pi + (1-\alpha) \mu_\texttt{rand}$.

\subsection{Additional Experimental Results}

In Figure \ref{fig:estims}, we present comparison of the two naive estimators, and the proposed estimators (with and without CDIS) on three domains.
For the recommender system, the horizon length is one and hence our estimator with and without CDIS behave exactly the same.
We see a similar behavior for the diabetes treatment, since the output of the policy corresponds to the parameters of another  insulin controlling policy \citep{bastani2014model}, which makes the horizon effectively of length one.
The variance reduction benefit of CDIS can be observed in the Gridworld setting, which has a longer horizon.

The output of the naive estimator $\hat \sigma^{2!!}_n$ was outside the limits of the plotted y-axis for all the domains and hence it is not visible.
The output of the other naive estimator $\hat \sigma^{2!}_n$ is nearly unbiased for the recommender systems and the Gridworld domains, both of which have discrete actions.
For the diabetes treatment domain, it can be observed that $\hat \sigma^{2!}_n$ results in biased estimates when the number of samples are small.

\begin{figure}
    \centering
    \includegraphics[width=0.32\textwidth]{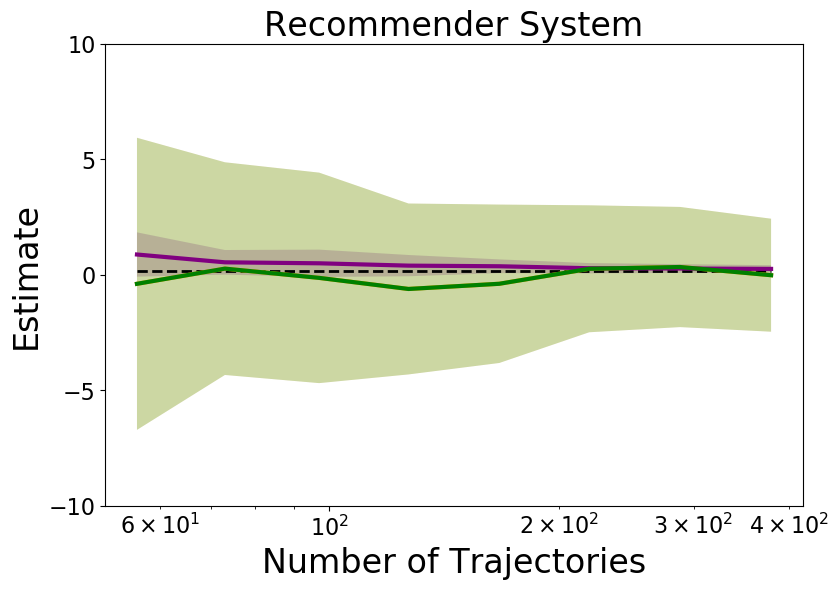}
    \includegraphics[width=0.32\textwidth]{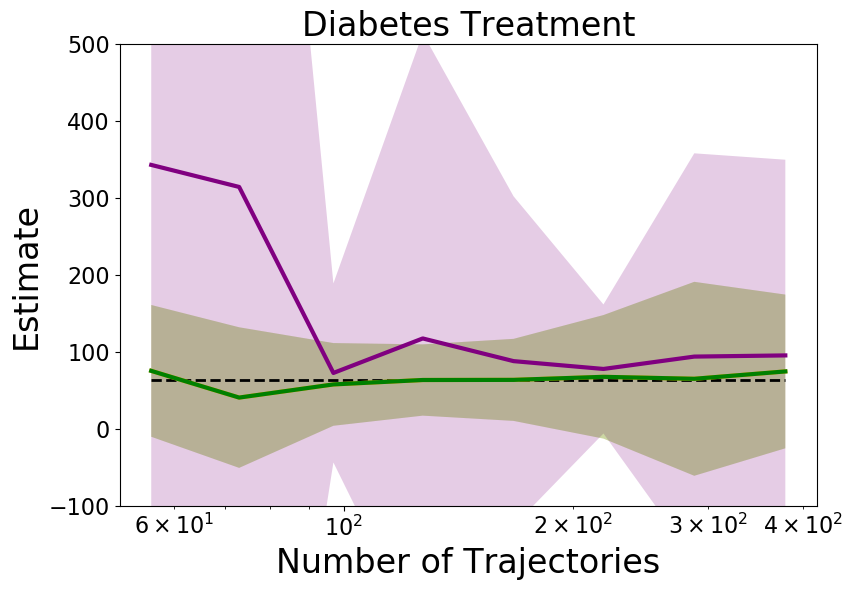}
    \includegraphics[width=0.32\textwidth]{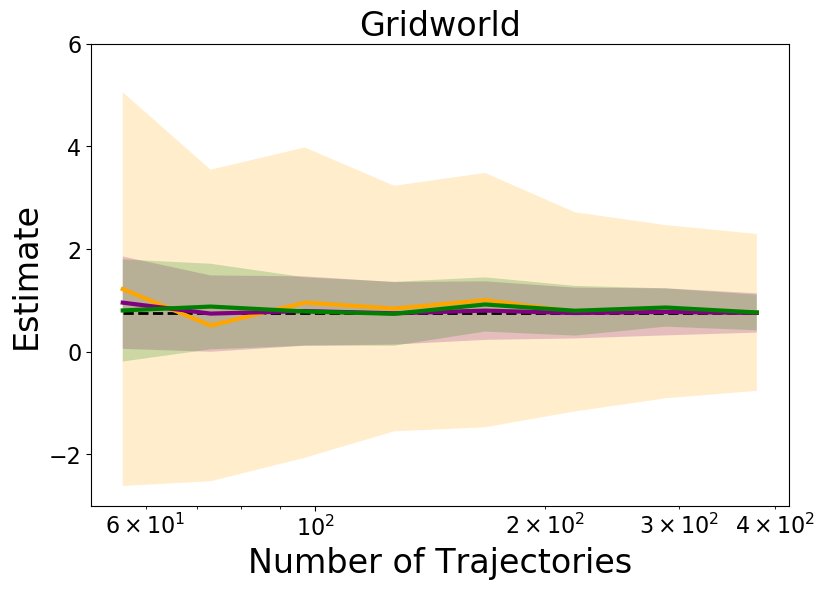}
    \\
    \includegraphics[width=0.6\textwidth]{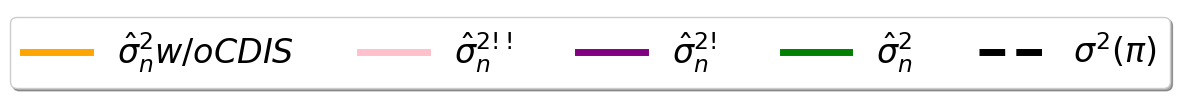}
    \caption{Comparison of different estimators on three domains. 
    The estimate value for $\sigma^{2!!}$ was outside the limits of the y-axis for all the domains and hence is not visible. 
    For the recommender system and the diabetes treatment domain, proposed estimators $\hat \sigma_n^2$, with and without coupled-decision importance sampling (CDIS), overlap.
    Results were averaged over 100 trials and the shaded regions correspond to one standard deviation.}
    \label{fig:estims}
\end{figure}